\documentclass{article}
\usepackage{graphicx}
\usepackage{amssymb}
\usepackage{amsmath}
\usepackage{amsthm}
\usepackage[square, numbers]{natbib}
\usepackage{placeins}
\usepackage{xcolor}
\usepackage{algorithm}
\usepackage{algpseudocode}
\usepackage{comment}
\usepackage{url}

\usepackage[hidelinks]{hyperref}
\newtheorem{definition}{Definition}
\newtheorem{lemma}{Lemma}
\newtheorem{theorem}{Theorem}

\allowdisplaybreaks
\begin{document}

\vspace*{1em}

\begin{center}
{\Large\bf 
T.W.H.\ Neuteboom
} 

\vspace{2em}
 
{\LARGE\bf 
Modifying Squint for
\\
Prediction with Expert Advice
\\ \vspace*{.25em}
in a Changing Environment
} 

 \vspace{.25em} 

\vspace{10em} 

{\large\bf 
Bachelor Thesis
} 

\vspace{1em}

{\large\bf 
16 June 2020
}

\vspace{10em} 

{\large\bf
\begin{tabular}{ll}
Thesis Supervisor: & Dr.\ T.A.L.\ van Erven
\end{tabular}
}

\vfill

\includegraphics{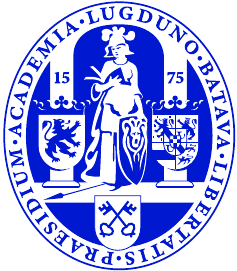}\\

\vspace{2em}

{\large\bf 
Leiden University\\
Mathematical Institute\\
}
\end{center}

\pagenumbering{gobble}

\newpage
\pagenumbering{arabic}
\tableofcontents

\newpage
\section{Introduction}
Online learning is an important domain in machine learning which is concerned with a learner predicting a sequence of outcomes and improving the predictions as time continues by learning from previous outcomes. Every round a prediction is made after which the outcome is revealed. Using the knowledge about this outcome the strategy for predicting the outcome in the next round can be adjusted. The learner wants to maximize the amount of good predictions made \cite{Backgr1}. \par
In a specific case of online learning the learner receives advice from $K$ so-called experts. It maintains a weight vector which is a probability distribution over the experts. Every round the learner randomly picks an expert using this distribution and follows its advice. When the outcome is revealed, the learner knows which experts were wrong and the weight vector will be adjusted accordingly to improve its chances of making a good prediction in the next round \cite{Backgr2}. This also is the form of online learning which we will study in this thesis. \par
The main question in this field of research is: what is the best way to adjust the weight vector? Multiple algorithms have been designed to determine the weight vector, but we will specifically look at the Squint algorithm \cite{Squint}. This algorithm is designed to always function at least as well as other known algorithms and in specific cases it functions much better. Squint was made for a non-changing environment. This means that it was designed for a setting in which the probability of a certain expert making a good prediction will not change over time. This thesis is concerned with making Squint function well in a changing environment, i.e.\ where experts can start performing better or worse at some point.\par
In Chapter \ref{sec:2} we will introduce the mathematical setting, present two algorithms, including Squint, for the non-changing environment and compare their properties. Then in Chapter \ref{sec:3} we will analyze how algorithms for the non-changing environment are usually made suitable for the changing environment. However, this conventional method makes the desired properties Squint has in a non-changing environment vanish. In order to find a way to retain Squint's properties, we will first dive into the design of the Squint algorithm and the proof of its properties in Chapter \ref{sec:4}. Finally, in Chapter \ref{sec:5} we will make our contribution by combining all the gathered information, designing a method which makes Squint function well in a changing environment and making sure its properties are preserved. We summarize our results in the concluding Chapter \ref{sec:6}.

\newpage
\section{Prediction with Expert Advice}\label{sec:2}
In this thesis we will study the Squint algorithm which is used for a specific case of online learning \cite{Backgr1}. We will now introduce the setting and make some definitions as used in the article of the Squint algorithm \cite{Squint}.

\subsection{Setting}\label{subsec:2.1}
Each round $t=1,2,\ldots,T$ a learner wants to predict an outcome $y_t$, which is determined by the environment. The loss of a prediction $p_t$ is denoted by $l(p_t,y_t) \in [0,1]$ and indicates how good the prediction was, where a low loss resembles a good prediction and a high loss resembles a bad one. For example, $y_t\in\{0,1\}$ can indicate whether it rains or not on day $t$. The learner predicts a probability $p_t\in[0,1]$ of it raining that day. The loss then can be defined as $l(p_t,y_t)=|p_t-y_t|$.\par
Each round the learner obtains advice from $K$ experts. For each expert $k\in\{1,2,\dots,K\}$ this advice is in the form of a prediction $p_t^k$ of the outcome $y_t$. However, their predictions are not necessarily right. Hence, every expert $k$ suffers a loss $l_t^k := l(p_t^k, y_t)$. Based on the experts' losses in previous rounds the learner will decide which expert's advice to adopt. The learner's decisions are randomized. Thus, they are made using a probability vector $\mathbf{w}_t$ (the components are non-negative and add up to 1) on the experts, which can be adjusted every round by the learner.\par
We now define $\mathbf{l}_t:=(l_t^1,l_t^2,\dots,l_t^K)^{\top}\in[0,1]^K$ as the losses of the $K$ experts in round $t$. Then the dot product $\mathbf{w}_t^{\top}\mathbf{l}_t$ resembles the expected loss of the learner. Since a low loss induces a good performance, we define the learner's performance compared to expert $k$ by $r_t^k := \mathbf{w}_t^{\top}\mathbf{l}_t-l_t^k$. This value resembles how much the learner is expected to regret using his probability distribution on the experts instead of deterministically picking expert $k$ and hence is called the instantaneous regret compared to expert $k$. Finally, this leads us to defining the total regret by
\begin{equation} \label{eq:RTk}
R_T^k := \sum_{t=1}^Tr_t^k = \sum_{t=1}^T\big(\mathbf{w}_t^{\top}\mathbf{l}_t-l_t^k\big).
\end{equation}
The goal of the learner is to perform as well as the best expert or more specifically, the expert with the lowest accumulated loss. Hence, the goal is to have `small' regret $R_T^k$ simultaneously for all experts $k$ after any number of rounds $T$. The total number of rounds $T$ is known to the learner before it starts the learning task and thus can be used when determining $\mathbf{w}_t$.\par
One could question when we can call the regret `small'. We assume that every expert will always have the same probability of making a good prediction. This means that an expert's total loss will grow linearly over time. We want our total regret to grow slower such that the average instantaneous regret approaches zero as the amount of time steps increases. For this reason we call the regret `small' if it grows sublinearly in $T$. We will now look into an algorithm which guarantees this property for the regret.

\subsection{Hedge Algorithm}
The question we aim to answer is: how should the learner adjust the probability vector $\mathbf{w}_t$ each round in order to have `small' regret? One way to do this is described by Freund and Schapire in their so-called Hedge algorithm \cite{Hedge}. This algorithm depends on a parameter $\eta \in [0,\infty)$, which can be chosen by the user. We define $w_t^k$ to be the $k$'th component of the probability vector $\mathbf{w}_t$, which in essence equals the weight put on expert $k$. Then the Hedge algorithm determines the probability vector for the next round by, for each $k$, setting
\[w_{t+1}^k = \frac{\pi(k) e^{-\eta\sum_{s=1}^tl_s^k}}{\sum_{i=1}^K \pi(i) e^{-\eta\sum_{s=1}^tl_s^i}}.\]
Here $\pi(k)$ is the prior distribution on the experts. So $\pi(k)$ equals the weight put on expert $k$ before the algorithm starts. When we set $\pi(k)$ to be the uniform distribution, so $\pi(k) = \frac{1}{K}$ for all $k$, and we set $\eta = \sqrt{\frac{8}{T}\ln{K}}$, Freund and Schapire show that the algorithm bounds the total regret by
\[R_T^k \leq \sqrt{\frac{T}{2}\ln{K}}\]
for each total number of rounds $T$ and every expert $k$. Other algorithms also obtain explicit bounds like we have here for Hedge. However, those bounds look more complex. Hence, from now on we use the following definition:
\begin{definition}
\label{def:curlybound}
Let $D \subseteq \mathbb{R}^n$ be the domain of the functions $f:D \rightarrow \mathbb{R}$ and $g:D \rightarrow \mathbb{R}$. We write $f(x) \preccurlyeq g(x)$ if there exists an absolute constant $c\in\mathbb{R}_{>0}$ such that for all $x\in D$ we have $f(x) \leq cg(x)$.
\end{definition}\noindent
With this definition we can write a cleaner expression for the bound obtained by the Hedge algorithm:
\begin{equation}\label{eq:Hedgebound}
R_T^k \preccurlyeq \sqrt{T\ln{K}}
\end{equation}
Thus, the Hedge algorithm guarantees that the total regret grows sublinearly over time, which is exactly what we wished for. So using this algorithm we obtain `small' regret. But there exist algorithms which yield even smaller regret, like is shown in the next section.

\subsection{Squint Algorithm}\label{subsec:2.3}
More recently, Koolen and Van Erven have come up with another algorithm, named Squint \cite{Squint}. This algorithm makes use of $V_T^k := \sum_{t=1}^T(r_t^k)^2$, the cumulative uncentered variance of the instantaneous regrets. Furthermore, it involves a learning rate $\eta \in [0,\frac{1}{2}]$. However, the optimal learning rate is not known to the learner. So it uses a so-called prior distribution $\gamma(\eta)$ on the learning rate. Finally, the prior distribution on the experts is again denoted by $\pi(k)$. Using these distributions, Squint determines the probability vector for the next round by setting
\begin{equation}\label{eq:Squint}
\mathbf{w}_{t+1} = \frac{\mathbb{E}_{\gamma(\eta)\pi(k)}\big[e^{\eta R_t^k-\eta^2V_t^k}\eta \mathbf{e}_k\big]}{\mathbb{E}_{\gamma(\eta)\pi(k)}\big[e^{\eta R_t^k-\eta^2V_t^k}\eta\big]}
\end{equation}
where $\mathbf{e}_k$ is the unit vector in the $k$'th direction.\par
For different choices of $\gamma(\eta)$ they obtain bounds for $R_T^{\mathcal{K}} := \mathbb{E}_{\pi(k|\mathcal{K})}R_T^k$ with $\mathcal{K} \subseteq \{1,2,\dots,K\}$ a subset of the set of experts and $\pi(k|\mathcal{K})$ the prior $\pi(k)$ conditioned on $\mathcal{K}$. This is the expected total regret compared to experts in the subset $\mathcal{K}$. Likewise, they define $V_T^{\mathcal{K}} := \mathbb{E}_{\pi(k|\mathcal{K})}V_T^k$. For a certain choice of $\gamma(\eta)$, on which we will go into further detail in Chapter \ref{sec:4}, they obtain the bound
\begin{equation}\label{eq:SquintRTk}
R_T^{\mathcal{K}} \preccurlyeq \sqrt{V_T^{\mathcal{K}}\big(\ln(\ln{T})-\ln{\pi(\mathcal{K})}\big)} + \ln(\ln{T})-\ln{\pi(\mathcal{K})}
\end{equation}
for each subset of experts $\mathcal{K}$ and every number of rounds $T$. Here $\pi(\mathcal{K})$ is the weight put on the elements in the subset $\mathcal{K}$ by the prior distribution. Usually $\mathcal{K}$ is chosen to be the set, unknown to the algorithm, of the best experts, such that the bound tells us something about the expected regret compared to the best experts. If this regret is `small', then the regret compared to the worse experts definitely is small as well. However, we now do not have a guarantee of our regret compared to the very best expert, which is what we initially were after.\par
Hence, one might wonder what makes this bound better than the one (\ref{eq:Hedgebound}) of the Hedge algorithm. Instead of using the time in $\sqrt{T}$ the new bound involves the variance in $\sqrt{V_T^{\mathcal{K}}}$. Since the instantaneous regret is bounded by $r_t^k \in [-1,1]$, we conclude that $V_T^k = \sum_{t=1}^T(r_t^k)^2 \leq T$. Often the variance is strictly smaller than $T$, which makes it a stricter bound. On top of that, the factor $\ln(K)$ is replaced by $-\ln{\pi(\mathcal{K})}$. When the prior distribution $\pi(k)$ is chosen uniformly, this factor is again smaller. Lastly, the factor $\ln(\ln{T})$ practically behaves like a small constant and thus can be neglected.\par
In the worst case there are no multiple good experts and we have to choose $\mathcal{K}$ to be a single expert. When using the uniform prior, the term $-\ln\pi(\mathcal{K})$ equals $\ln K$. Also, in the worst case $V_T^{\mathcal{K}}$ will be close to $T$. Since $\ln(\ln T)$ is negligible, the Squint bound now matches the bound of Hedge in some sense. But again, this holds for the worst case. In practice there are often multiple good experts and the expected variance $V_T^{\mathcal{K}}$ is much smaller than $T$ for specific cases \cite{SquintUseful}. That makes this bound significantly better than the Hedge bound. Another advantage is that the bound holds for every prior distribution $\pi(k)$ on the experts and not only for the uniform prior. For non-uniform priors it could also yield favorable bounds.\par
In Section \ref{subsec:2.1} we assumed that every expert always has the same probability of making a good prediction in order to define what we mean by `small' regret. However, in reality experts might notice over time that their predictions are not very accurate. This could lead to them changing their strategy of predicting and hence they could start making better or worse predictions. Moreover, if the sequence of outcomes changes this could also lead to certain experts gaining an advantage and making better predictions. How do we then define `small' regret and how can we make the Squint algorithm adapt to changes in performance of the experts? After all, we want our algorithm to put the most weight on the best experts, but when one expert performed poorly at first and now starts performing very well, we want the algorithm to forget about the first few bad predictions. So even though the current regret bound still holds, we now want to bound the total regret obtained since the last time an expert's performance changed and not necessarily since $t=1$. In the remainder of this thesis, we will focus on these problems caused by varying performances of the experts.

\newpage
\section{Changing Environment}\label{sec:3}
In this chapter we will be looking at a changing environment for the prediction with expert advice setting, i.e.\ the experts' performances change over time. Jun et al.\ \cite{ChEnv} studied this with a tool named coin betting. We will use their research as a starting point and apply it to our learning task.

\subsection{(Strongly) Adaptive Algorithms}
When the environment changes, we want the learner to adapt to this change. Originally, the probability vector $\mathbf{w}_t$ was based on the data from the first round $t=1$ until the current round. However, when the environment has changed, we want the learner to forget about the data from before that point. So we only want to use data from an interval, starting at the last time the environment changed and ending at the current round. Applying this approach, we now look for algorithms which have `small' regret on every possible interval, since we do not know when the environment changes and thus which intervals to look at. First, we will adjust our definition of the total regret (\ref{eq:RTk}) to this situation. We define the total regret on a contiguous interval $I=[I_1,I_2]:=\{I_1,I_1+1,\dots,I_2\}$ with $I\subseteq\{1,2,\dots,T\}$ compared to expert $k$ by
\begin{equation}\label{eq:stregretint}
R_I^k := \sum_{t=I_1}^{I_2}r_t^k=\sum_{t=I_1}^{I_2}(\mathbf{w}_t^{\top}\mathbf{l}_t-l_t^k).
\end{equation}
We want an algorithm which makes $R_I^k$ grow sublinearly over time on its interval. Often used definitions are the following: we call an algorithm adaptive to a changing environment if $R_I^k$ grows with $\sqrt{T}$ for every contiguous $I\subseteq\{1,2,\dots,T\}$ and strongly adaptive if it grows with $\sqrt{|I|\ln{T}}$ where $|I|=I_2-I_1+1$ is the length of the interval. Note that the stopping time $T$ is known beforehand and hence can be used by the algorithm, but the interval $I$, on which we measure the regret, is not known and cannot be used by the algorithm. This is because we want to have `small' regret on all possible intervals $I$. \par
If we knew beforehand when the environment changes, we could apply Hedge or Squint to every separate interval between the times the environment changes. However, we do not know when the environment changes and thus cannot tell what the starting points of these intervals are. That is where a so-called meta algorithm which learns these starting points comes in.

\subsection{Meta Algorithms}\label{sec:meta-alg}
The idea of a meta algorithm is that it uses so-called black-box algorithms like the Hedge algorithm. For each possible starting point a black-box algorithm is introduced. These algorithms compute $\mathbf{w}_t$ for every time step of their intervals and can only use data of the environment available from their starting points onwards. In our case the only difference between the black-box algorithms is their interval. We choose the type of algorithm (e.g.\ Hedge or Squint) to be the same for all black-boxes.\par
The meta algorithm then keeps track of a probability vector $\mathbf{q}_t$ on the active black-box algorithms (i.e.\ the algorithms that produce an output at time $t$). It uses this probability vector to randomly determine which black-box algorithm it should follow. Based on knowledge from previous rounds the probability vector is adjusted. This is actually very similar to our prediction with expert advice setting where the black-box algorithms are the experts and the meta algorithm is the learner.\par
When using this meta algorithm we should reformulate our regret on an interval $I$. Let $\mathcal{B}=\{1, 2, \dots, B\}$ be the set of all black-box algorithms and let $\mathcal{B}_t\subseteq\mathcal{B}$ be the set of active black-box algorithms at time $t$. For the algorithm $b\in\mathcal{B}_t$ we define $\mathbf{w}_t^b$ as its computed probability vector at time $t$. Similar to (\ref{eq:stregretint}), the regret for this black-box algorithm on the interval $I$ is denoted by $R_I^k(b) := \sum_{t\in I}\big((\mathbf{w}_t^b)^{\top}\mathbf{l}_t-l_t^k\big)$. Finally, $\mathbf{q}_t$ is the probability vector with components $q_t^b$ on the set of black-box algorithms $\mathcal{B}$, which is determined by the meta algorithm $\mathcal{M}$. It has the properties $q_t^b\geq0$ if $b\in\mathcal{B}_t$, $q_t^b=0$ if $b\notin\mathcal{B}_t$ and $\sum_{b\in\mathcal{B}}q_t^b=1$. Since the regret is the difference between the expected loss of the learner and the loss of the expert, we conclude that the regret of the meta algorithm in combination with the black-box algorithms is given by
\[R_I^k(\mathcal{M}(\mathcal{B})):=\sum_{t\in I}\Big(\sum_{b\in\mathcal{B}_t}q_t^b\big(\mathbf{w}_t^b\big)^{\top}\mathbf{l}_t-l_t^k\Big).\]
Before we go into the precise formulation of the algorithm, we would like to look at the computational complexity. When we introduce a black-box algorithm for every possible starting point, we would have to keep track of a lot of algorithms. At time $t$ there would be $t$ active algorithms (one for each possible starting point, including the current round), which would sum to $\sum_{t=1}^Tt = \Theta(T^2)$. So the computation time would scale quadratically with $T$, whereas we want it to scale linearly with $T$, which is the case for algorithms in a non-changing environment. In order to reduce the computation time, we will only look at the so-called geometric covering intervals. Using these we will take less than $t$ black-box algorithms into account when at time $t$.

\subsection{Geometric Covering Intervals}
\FloatBarrier
Daniely et al.\ \cite{ChEnvPr} prove Lemma \ref{lem:covint} below, which states that every possible contiguous interval $I\subseteq\mathbb{N}$ can be partitioned into the geometric covering intervals. Hence, if we can guarantee a `small' sum of regrets on these specific intervals, we can guarantee a `small' regret on every possible interval. We will now make this more precise.\par
For every $n\in\mathbb{N}\cup\{0\}$, define $\mathcal{J}_n := \big\{[i\cdot2^n,(i+1)\cdot2^n-1]:i\in\mathbb{N}\big\}$ to be the collection of intervals of length $2^n$ with starting points $i\cdot2^n$. The geometric covering intervals are
\[\mathcal{J}:=\bigcup\limits_{n\in\mathbb{N}\cup\{0\}}\mathcal{J}_n.\]
As described by Jun et al.\ \cite{ChEnv}, ``\textit{$\mathcal{J}$ is the set of intervals of doubling length, with intervals of size $2^n$ exactly partitioning the set $\mathbb{N}\setminus\{1,2,\dots,2^n-1\}$}''. They also give a visualization of this set, displayed in Figure \ref{fig:JVis}.
\noindent
\FloatBarrier
\begin{figure}[h]
    \centering
    \includegraphics[scale=0.4]{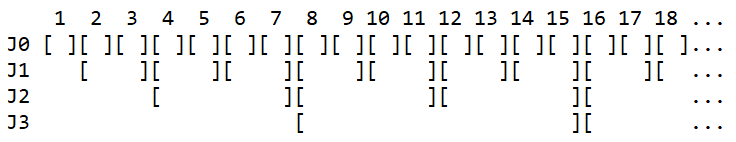}
    \caption{Geometric covering intervals. Each interval is denoted by [ ]. Adopted from \cite{ChEnv}.}
    \label{fig:JVis}
\end{figure}
\FloatBarrier
\noindent
Since $t$ is an element of an interval $J\in\mathcal{J}_n$ only if $n\leq\lfloor^2\log{t}\rfloor$, we see that at time $t$ there are $1+\lfloor^2\log{t}\rfloor$ active intervals. We will only use black-box algorithms which function solely on these intervals and not outside of them. So at time $t$ we also have $1+\lfloor^2\log{t}\rfloor$ active black-box algorithms, which sums to $\sum_{t=1}^T\big(1+\lfloor^2\log{t}\rfloor\big) = \Theta(T\ln{T})$. This improves the previous quadratic scaling $T^2$ to nearly linear in $T$, except for the minor overhead of an extra logarithmic factor. Thus, using the geometric covering intervals requires considerably less computation time than looking at all possible intervals with endpoint $t$. As mentioned before, the following lemma allows us to use the geometric covering intervals:

\begin{lemma}[{\cite[Lemma 1.2]{ChEnvPr}}]
\label{lem:covint}
Let $I=[I_1,I_2]\subseteq \mathbb{N}$ be an arbitrary interval. Then $I$ can be partitioned into a finite number of disjoint and consecutive intervals denoted $J^{(-c)},J^{(-c+1)},\dots,J^{(0)}, J^{(1)},\dots,J^{(d-1)},J^{(d)}$ with $c,d\geq0$ such that $J^{(i)}\in\mathcal{J}$ for all $i\in[-c,d]$ and such that
\[|J^{(i-1)}|/|J^{(i)}|\leq1/2\text{ for }i=-c+1,\dots,0\]
and \[|J^{(i+1)}|/|J^{(i)}|\leq1/2\text{ for }i=1,\dots,d-1.\]
\end{lemma}\noindent
The partitioning from the lemma is a sequence of smaller intervals which successively double and then successively halve in length. With this partitioning we can now decompose the total regret compared to an expert $k$, obtained by the meta algorithm with the black-box algorithms. To do so, we introduce the notation $b_J$ for the black-box algorithm which operates only on the interval $J$ and not outside of it.

\begin{align*}
R_I^k\big(\mathcal{M}(\mathcal{B})\big) & = \sum_{t\in I}\Big(\sum_{b\in\mathcal{B}_t}q_t^b\big(\mathbf{w}_t^b\big)^{\top}\mathbf{l}_t-l_t^k\Big) \\
 & = \sum_{i=-c}^d\sum_{t\in J^{(i)}}\Big(\sum_{b\in\mathcal{B}_t}q_t^b\big(\mathbf{w}_t^b\big)^{\top}\mathbf{l}_t -\big(\mathbf{w}_t^{b_{J^{(i)}}}\big)^{\top}\mathbf{l}_t + \big(\mathbf{w}_t^{b_{J^{(i)}}}\big)^{\top}\mathbf{l}_t-l_t^k\Big) \\
 & = \sum_{i=-c}^d\sum_{t\in J^{(i)}}\Big(\sum_{b\in\mathcal{B}_t}q_t^b\big(\mathbf{w}_t^b\big)^{\top}\mathbf{l}_t-\big(\mathbf{w}_t^{b_{J^{(i)}}}\big)^{\top}\mathbf{l}_t\Big) \\
 & \qquad + \sum_{i=-c}^d\sum_{t\in J^{(i)}}\Big(\big(\mathbf{w}_t^{b_{J^{(i)}}}\big)^{\top}\mathbf{l}_t-l_t^k\Big) \\
 & = \sum_{i=-c}^d\sum_{t\in J^{(i)}} \text{instantaneous regret of }\mathcal{M}\text{ compared to black-box }b_{J^{(i)}} \\
 & \qquad + \sum_{i=-c}^d\sum_{t\in J^{(i)}} \text{instantaneous regret of }b_{J^{(i)}}\text{ compared to expert }k \\
 & = \sum_{i=-c}^d \text{total regret on }J^{(i)}\text{ of }\mathcal{M}\text{ compared to black-box }b_{J^{(i)}} \\
 & \qquad + \sum_{i=-c}^d \text{total regret on }J^{(i)}\text{ of }b_{J^{(i)}}\text{ compared to expert }k \\
\end{align*}
From now on we use the following terms for the different regrets: the $\mathcal{M}$-regret (denoted $R_I^b(\mathcal{M})$) is the total regret on $I$ of $\mathcal{M}$ compared to a black-box $b$. The $b$-regret (denoted $R_I^k(b)$) is the total regret on $I$ of $b$ compared to expert $k$. Finally, the $\mathcal{M}(\mathcal{B})$-regret (denoted $R_I^k(\mathcal{M}(\mathcal{B}))$) is the total regret on $I$ of $\mathcal{M}$ with its black-box algorithms $\mathcal{B}$ compared to expert $k$. This gives the following equation:
\begin{equation}\label{eq:RIkdec}
R_I^k\big(\mathcal{M}(\mathcal{B})\big) = \sum_{i=-c}^d \Big( R_{J^{(i)}}^{b_{J^{(i)}}}(\mathcal{M}) + R_{J^{(i)}}^k(b_{J^{(i)}})\Big)
\end{equation}
So in order to have `small' regret on every interval $I$, we need the following properties:
\begin{enumerate}
\item the sum over the intervals $J^{(i)}$ of the $b_{J^{(i)}}$-regret must be `small' for all combinations of the $J^{(i)}$
\item the sum over the intervals $J^{(i)}$ of the $\mathcal{M}$-regret must be `small' for all combinations of the $J^{(i)}$
\end{enumerate}
We will first go into further detail on the meta algorithm that uses the geometric covering intervals, after which we will show these properties are satisfied.

\subsection{CBCE Algorithm}
We will now introduce the meta algorithm, called Coin Betting for Changing Environment (CBCE), which was established by Jun et al.\ \cite{ChEnv} and based on the work of Orabona and Pál \cite{CoinB}. We will present a slight adjustment of this algorithm such that it matches our notation. However, the functioning and corresponding properties stay the same.\par
For $t\in J = [J_1,J_2]$, let $r_t^{b_J}(\mathcal{M}) := \sum_{b\in\mathcal{B}_t}q_t^b\big(\mathbf{w}_t^b\big)^{\top}\mathbf{l}_t-\big(\mathbf{w}_t^{b_J}\big)^{\top}\mathbf{l}_t$ be the instantaneous regret of $\mathcal{M}$ compared to $b_J$. Let $\tau(b)$ be a prior distribution on the set of all (geometric covering interval) black-box algorithms $\{b_J : J\in\mathcal{J}\}$ and let $\tau_t(b)$ be the prior $\tau(b)$ restricted to all active black-box algorithms at time $t$. So $\tau_t(b_J) = \tau(b_J)/\tau(\mathcal{B}_t)$ if $t\in J$ and $\tau_t(b_J)=0$ if $t\notin J$ where $\tau(\mathcal{B}_t)$ is the weight put on the active algorithms by $\tau$. Then each round the probability vector $\mathbf{q}_t$ on these black-box algorithms is computed by
\[\mathbf{q}_t = \left\{\begin{matrix}\hat{\mathbf{q}}_t/||\hat{\mathbf{q}}_t||_1 && \text{if } ||\hat{\mathbf{q}}_t||_1 > 0 \\ \tau_t && \text{if } ||\hat{\mathbf{q}}_t||_1=0\end{matrix}\right.\]
where
\[\hat{q}_t^{b_J} = \left\{\begin{matrix} \tau(b_J)\cdot\max\{v_t^{b_J},0\} && \text{if } b_J\in\mathcal{B}_t \\ 0 && \text{if } b_J\notin\mathcal{B}_t\end{matrix}\right.,\]
\[v_t^{b_J} = \frac{1}{\big|[J_1,t]\big|}\Big(\sum_{i=J_1}^{t-1}g_i^{b_J}\Big)\cdot\Big(1+\sum_{i=1}^{t-1}z_i^{b_J}v_i^{b_J}\Big),\]
\[z_t^{b_J} = \left\{\begin{matrix} \sum_{i=1}^{t-1}g_i^{b_J} && \text{if } b_J\in \mathcal{B}_t \\ 0 && \text{if } b_J\notin \mathcal{B}_t \end{matrix}\right.\]
and \[g_t^{b_J} = \left\{\begin{matrix}r_t^{b_J}(\mathcal{M}) && \text{if } v_t^{b_J} > 0 \\ \max\big\{r_t^{b_J}(\mathcal{M}),0\big\} && \text{if } v_t^{b_J} \leq 0 \end{matrix}\right..\]
To clarify the functioning of the algorithm, we will also describe it in pseudocode (see Algorithm \ref{alg:CBCE}).
\begin{algorithm}[!ht]
\caption{Coin Betting for Changing Environment, adopted from \cite{ChEnv}.}\label{alg:CBCE}
\begin{algorithmic}
\State \textbf{Input:} A set of black-box algorithms $\mathcal{B} = \{b_J:J\in\mathcal{J}\}$ and a prior distribution $\tau(b)$ on $\mathcal{B}$.
\For{$t=1$ \textbf{to} $T$}
	\For{$b_J\in\mathcal{B}_t$}
		\State $z_t^{b_J} \gets \left\{\begin{matrix} \sum_{i=1}^{t-1}g_i^{b_J} && \text{if } b_J\in \mathcal{B}_t \\ 0 && \text{if } b_J\notin \mathcal{B}_t \end{matrix}\right.$
		\State $v_t^{b_J} \gets \frac{1}{|[J_1,t]|}\big(\sum_{i=J_1}^{t-1}g_i^{b_J}\big)\cdot\big(1+\sum_{i=1}^{t-1}z_i^{b_J}v_i^{b_J}\big)$
	\EndFor
	\For{$b_J\in\mathcal{B}$}
		\State $\hat{q}_t^{b_J} \gets \left\{\begin{matrix} \tau(b_J)\cdot\max\{v_t^{b_J},0\} && \text{if } b_J\in\mathcal{B}_t \\ 0 && \text{if } b_J\notin\mathcal{B}_t\end{matrix}\right.$
	\EndFor
	\State $\mathbf{q}_t \gets \left\{\begin{matrix}\hat{\mathbf{q}}_t/||\hat{\mathbf{q}}_t||_1 && \text{if } ||\hat{\mathbf{q}}_t||_1 > 0 \\ \tau_t && \text{if } ||\hat{\mathbf{q}}_t||_1=0\end{matrix}\right.$
	\State Obtain the losses $\mathbf{l}_t$ from the environment
	\For{$b_J\in\mathcal{B}_t$}
		\State Obtain the probability vectors $\mathbf{w}_t^{b_J}$ from the black-box algorithms
		\State $r_t^{b_J}(\mathcal{M}) \gets \sum_{b\in\mathcal{B}_t}q_t^b\big(\mathbf{w}_t^b\big)^{\top}\mathbf{l}_t-\big(\mathbf{w}_t^{b_J}\big)^{\top}\mathbf{l}_t$
		\State $g_t^{b_J} \gets \left\{\begin{matrix}r_t^{b_J}(\mathcal{M}) && \text{if } v_t^{b_J} > 0 \\ \max\big\{r_t^{b_J}(\mathcal{M}),0\big\} && \text{if } v_t^{b_J} \leq 0 \end{matrix}\right.$
	\EndFor
\EndFor
\end{algorithmic}
\end{algorithm}
The creators prove a bound on the $\mathcal{M}$-regret of CBCE for the following choice of prior:
\[\tau(b_J) = Z^{-1}\cdot\frac{1}{J_1^2\big(1+\lfloor ^2\log{J_1}\rfloor\big)} \text{ for all } J\in\mathcal{J}\]
Here $Z$ is a normalization factor. The bound, given on interval $J\in\mathcal{J}$, is
\[R_J^{b_J}(\mathcal{M}) \leq \sqrt{|J|\big(7\ln(J_2)+5\big)} \preccurlyeq \sqrt{|J|\ln{J_2}}.\]
Using this result, we will now show that the two necessary properties in order to have `small' regret on an interval $I$ are satisfied when one applies Hedge for the black-box algorithms. This proof is also adopted from \cite{ChEnv}.

\subsection{Applying CBCE to Hedge}\label{sec:cbcehedge}
Using (\ref{eq:RIkdec}) we will bound the $\mathcal{M}(\mathcal{B})$-regret on an arbitrary interval $I=[I_1,I_2]$ compared to expert $k$ where we take CBCE as our meta algorithm and Hedge for the black-box algorithms. Since the $b_J$-regret on an interval $J$ is equal to the standard regret (\ref{eq:RTk}) for $T=|J|$, the $b_J$-regret for Hedge is bounded by
\[R_J^k(b_J)\leq\sqrt{\frac{|J|}{2}\ln{K}}.\]
This gives the bound
\begin{align*}
R_I^k\big(\mathcal{M}(\mathcal{B})\big) & = \sum_{i=-c}^d \Big( R_{J^{(i)}}^{b_{J^{(i)}}}(\mathcal{M}) + R_{J^{(i)}}^k(b_{J^{(i)}})\Big) \\
 & \leq \sum_{i=-c}^d \bigg(\sqrt{|J^{(i)}|\big(7\ln(J_2^{(i)})+5\big)} + \sqrt{\frac{|J^{(i)}|}{2}\ln{K}}\bigg).
\end{align*}
Since the intervals $J^{(i)}$ form the partitioning from Lemma \ref{lem:covint}, we know that $J^{(i)}\subseteq I$ and hence $J_2^{(i)} \leq I_2$ for all $i$. This gives
\begin{align*}
R_I^k\big(\mathcal{M}(\mathcal{B})\big) & \leq \sum_{i=-c}^d \bigg(\sqrt{|J^{(i)}|\big(7\ln(I_2)+5\big)} + \sqrt{\frac{|J^{(i)}|}{2}\ln{K}}\bigg) \\
 & = \bigg(\sqrt{7\ln(I_2)+5}+\sqrt{\frac{1}{2}\ln{K}}\bigg)\sum_{i=-c}^d\big|J^{(i)}\big|^{1/2}.
\end{align*}
We know that $|J^{(0)}|, |J^{(1)}| \leq |I|$ and that the lengths successively halve when $i$ goes down from 0 or up from 1. So we obtain
\begin{align*}
\sum_{i=-c}^d\big|J^{(i)}\big|^{1/2} & \leq 2\sum_{n=0}^{\infty}\big(2^{-n}|I|\big)^{1/2} \\
 & = \frac{2\sqrt{2}}{\sqrt{2}-1}|I|^{1/2}.
\end{align*}
Now we bound the regret by
\begin{align*}
R_I^k\big(\mathcal{M}(\mathcal{B})\big) & \leq \bigg(\sqrt{7\ln(I_2)+5}+\sqrt{\frac{1}{2}\ln{K}}\bigg) \cdot \frac{2\sqrt{2}}{\sqrt{2}-1}|I|^{1/2} \\
 & = \frac{2\sqrt{2}}{\sqrt{2}-1}\sqrt{|I|\big(7\ln(I_2)+5\big)}+\frac{2}{\sqrt{2}-1}\sqrt{|I|\ln{K}} \\
 & \preccurlyeq \sqrt{|I|\ln{I_2}}+\sqrt{|I|\ln{K}}.
\end{align*}
We conclude that the two properties required for a `small' regret on $I$ are satisfied. Since $I_2\leq T$, we find that CBCE combined with Hedge is a strongly adaptive algorithm. In the previous chapter we saw that Squint gave a better bound than Hedge. Consequently, we want to try to obtain a better bound in a changing environment by applying CBCE to Squint.

\subsection{Applying CBCE to Squint}\label{subsec:CBCESquint}
The creators of Squint \cite{Squint} obtained bounds for $R_T^{\mathcal{K}} = \mathbb{E}_{\pi(k|\mathcal{K})}R_T^k$, the expected regret compared to a subset $\mathcal{K}\subseteq\{1,2,\dots,K\}$ of experts. Likewise, we now define this regret on the interval $I$: $R_I^{\mathcal{K}} := \mathbb{E}_{\pi(k|\mathcal{K})}R_I^k$. Furthermore, we define $V_I^{k}:=\sum_{t=I_1}^{I_2}(r_t^k)^2$ and $V_I^{\mathcal{K}}:=\mathbb{E}_{\pi(k|\mathcal{K})}V_I^k$ similarly.\par
For the expected regret compared to a subset $\mathcal{K}$ we can make the same decomposition as displayed in equation (\ref{eq:RIkdec}). We then obtain
\begin{align*}
R_I^{\mathcal{K}}\big(\mathcal{M}(\mathcal{B})\big) & = \mathbb{E}_{\pi(k|\mathcal{K})}\Big[R_I^k\big(\mathcal{M}(\mathcal{B})\big)\Big] \\
 & = \mathbb{E}_{\pi(k|\mathcal{K})}\bigg[\sum_{i=-c}^d \Big( R_{J^{(i)}}^{b_{J^{(i)}}}(\mathcal{M}) + R_{J^{(i)}}^k(b_{J^{(i)}})\Big)\bigg] \\
 & = \sum_{i=-c}^d \Big( R_{J^{(i)}}^{b_{J^{(i)}}}(\mathcal{M}) + R_{J^{(i)}}^{\mathcal{K}}(b_{J^{(i)}})\Big).
\end{align*}
This is very similar to equation (\ref{eq:RIkdec}). The analysis of the sum over the first term stays the same as in Section \ref{sec:cbcehedge}. What is left is the analysis of the sum over the second term, which is different than before. To do this, we would like to know over how many intervals we are summing (i.e.\ what is $c+d+1$?). Lemma \ref{lem:covint} tells us that the $J^{(i)}$ first successively at least double and then successively at most halve in length. This enables us to bound the length of $I$ from below by
\begin{align*}
|I| & = |J^{(-c)}| + |J^{(-c+1)}| + \dots + |J^{(0)}| + |J^{(1)}| + \dots + |J^{(d-1)}| + |J^{(d)}| \\
 & \geq 1 + 2 + \dots + 2^{c} + 2^{d-1} + \dots +2 + 1 \\
 & = \sum_{i=0}^c2^i + \sum_{i=0}^{d-1}2^i \\
 & = 2^{c+1}-1+2^d-1.
\end{align*}
Defining $\alpha := \max\{c+1,d\}$ gives
\[2^{\alpha} \leq |I|+2\]
and hence
\[\alpha \leq {}^2\log(|I|+2).\]
This finally gives
\[c+d+1 \leq 2\alpha \leq 2\cdot {}^2\log(|I|+2).\]
Remember that the total regret for Squint in a non-changing environment is bounded by
\begin{equation}\tag{\ref{eq:SquintRTk}}
R_T^{\mathcal{K}} \preccurlyeq \sqrt{V_T^{\mathcal{K}}\big(\ln(\ln{T})-\ln{\pi(\mathcal{K})}\big)} + \ln(\ln{T})-\ln{\pi(\mathcal{K})}.
\end{equation}
We will now determine the bound for the sum of the Squint regret over the intervals $J^{(i)}$. In the summation over the intervals $J^{(i)}$, we first only look at the second and third term of the Squint regret.
\begin{align*}
\sum_{i=-c}^d\big(\ln(\ln {|J^{(i)}|})-\ln\pi(\mathcal{K})\big) & \leq \sum_{i=-c}^d\big(\ln(\ln {|I|})-\ln\pi(\mathcal{K})\big)\\
& \leq 2\cdot {}^2\log(|I|+2)\cdot \big(\ln(\ln{|I|})-\ln\pi(\mathcal{K})\big) \\
 & \preccurlyeq \ln|I|\cdot\big(\ln(\ln{|I|})-\ln\pi(\mathcal{K})\big)
 \end{align*}
In order to prove a good bound on the first term, we need the following lemma:
\begin{lemma}
\label{lem:rootineq}
For all $a_1, a_2, \dots, a_m \in \mathbb{R}_{\geq0}$ with $m\in\mathbb{N}$:
\[\sum_{i=1}^m\sqrt{a_i}\leq\sqrt{m\sum_{i=1}^ma_i}\]
\end{lemma}
\begin{proof}[\textbf{Proof}]
Let $P_u$ denote the uniform distribution on the set $\{a_1,a_2,\dots,a_m\}$. Since the square root is a concave function on $\mathbb{R}_{\geq0}$, we can use Jensen's inequality and obtain the following:
\begin{align*}
\sum_{i=1}^m\sqrt{a_i} & = m\sum_{i=1}^m\frac{1}{m}\sqrt{a_i} \\
 & = m\mathbb{E}_{P_u}\big[\sqrt{a}\big] \\
 & \leq m\sqrt{\mathbb{E}_{P_u}[a]} \\
 & = m\sqrt{\sum_{i=1}^m\frac{1}{m}a_i} \\
 & = \sqrt{m\sum_{i=1}^ma_i}
\end{align*}
\end{proof}
\noindent
Using this lemma, we will derive a bound on the first term of the Squint regret:
\begin{align*}
& \sum_{i=-c}^d\sqrt{V_{J^{(i)}}^{\mathcal{K}}\big(\ln(\ln{|J^{(i)}|})-\ln\pi(\mathcal{K})\big)} \\
& \qquad \qquad \leq \sqrt{\ln(\ln{|I|})-\ln\pi(\mathcal{K})}\cdot\sum_{i=-c}^d\sqrt{V_{J^{(i)}}^{\mathcal{K}}} \\
& \qquad \qquad \leq \sqrt{\ln(\ln{|I|})-\ln\pi(\mathcal{K})}\cdot\sqrt{2\cdot {}^2\log(|I|+2)\cdot\sum_{i=-c}^dV_{J^{(i)}}^{\mathcal{K}}}\\
& \qquad \qquad \leq \sqrt{\ln(\ln{|I|})-\ln\pi(\mathcal{K})}\cdot\sqrt{2\cdot {}^2\log(|I|+2)\cdot V_I^{\mathcal{K}}} \\
& \qquad \qquad \preccurlyeq \sqrt{\ln|I|\cdot V_I^{\mathcal{K}}\big(\ln(\ln{|I|})-\ln\pi(\mathcal{K})\big)}
\end{align*}
Conclusively, the bound of CBCE combined with Squint is
\begin{equation}\label{eq:RIKwbound}
\begin{aligned}
R_I^{\mathcal{K}}\big(\mathcal{M}(\mathcal{B})\big) & \preccurlyeq \sqrt{|I|\ln{I_2}} + \sqrt{\ln|I|\cdot V_I^{\mathcal{K}}\big(\ln(\ln{|I|})-\ln\pi(\mathcal{K})\big)}\\
& \qquad + \ln|I|\cdot\big(\ln(\ln|I|)-\ln\pi(\mathcal{K})\big).
\end{aligned}
\end{equation}
Remember that $V_I^{\mathcal{K}}\leq|I|$ and that $\ln(\ln|I|)$ can be neglected, since it practically behaves like a small constant. As $|I| = I_2-I_1+1 \leq I_2$ and thus $\ln{|I|} \leq \ln{I_2}$, we can conclude that the first term grows faster over the interval length $|I|$ than the second term. Moreover, because $\ln|I|$ grows slower than $\sqrt{|I|}$, the first term also grows faster than the third term. Hence, the CBCE regret dominates the Squint regret in this expression. Likewise, the CBCE regret dominates the Hedge regret in the expression from Section \ref{sec:cbcehedge}. So the advantage Squint has over Hedge in a non-changing environment vanishes in a changing environment.\par
Therefore, our goal is to find a way to retain the advantages of Squint when applying it in a changing environment. In the following chapters we will focus on adjusting Squint for this matter. We will first give a detailed proof of the regular bound, so we can later on build on the same techniques and construct a bound in a changing environment.

\newpage
\section{Squint}\label{sec:4}
After having discussed the effects of a changing environment, we will now go back to a non-changing environment. In order to make Squint work well under varying circumstances, we first have to go deeper into its properties for the non-changing case. Specifically, we will show a proof of its regret bound (\ref{eq:SquintRTk}) from Chapter \ref{sec:2}, as this will be an important stepping stone for proving a bound in a changing environment.

\subsection{Reduction to a Surrogate Task}
First, we will introduce a new task, the so-called surrogate task, to which we can reduce our original learning task, as done similarly by Van der Hoeven et al.\ \cite{ExpW}. Remember that, when we introduced Squint, we used a prior $\gamma(\eta)$ on the learning rate $\eta\in[0,\frac{1}{2}]$, since we did not know what the optimal learning rate was. For the surrogate task we try to find the best expert, but also the best learning rate. As a consequence, we now keep track of a probability distribution $P_t(\eta,k)$ on the learning rates $\eta\in[0,\frac{1}{2}]$ and the experts $k\in\{1,2,\dots,K\}$ instead of our previous probability vector $\mathbf{w}_t$. And instead of our previous loss $l_t^k$, which only depended on the expert, we now have a so-called surrogate loss $\hat{f}_t(\eta,k) := -\eta r_t^k + \eta^2 (r_t^k)^2$ which depends on the learning rate and the expert. Here, $r_t^k = \mathbf{w}_t^{\top}\mathbf{l}_t-l_t^k \in [-1,1]$ is the instantaneous regret compared to expert $k$, just like defined in Section \ref{subsec:2.1}. The definition of this surrogate loss will later turn out to be useful, since the sum of these losses yields the total regret $R_T^k$ and variance $V_T^k$. Next, we make some general definitions, which we will use multiple times in the remainder of this thesis:

\begin{definition}\label{def:mixloss}
Let $\Theta$ be a measurable space, let $P_t$ be a probability distribution on $\Theta$ for all $t$ and let $g_t(\theta)\in\mathbb{R}$ be the loss of $\theta\in\Theta$ at time $t$. Then the mix loss of $g_t$ under $P_t$ is defined by
\[L(g_t,P_t) := -\ln\mathbb{E}_{P_t}\big[e^{-g_t(\theta)}\big].\]
\end{definition}
\begin{definition}\label{def:surreg}
Let $\Theta$ be a measurable space, let $Q$ and $P_t$ be probability distributions on $\Theta$ for all $t$ and let $g_t(\theta)\in\mathbb{R}$ be the loss of $\theta\in\Theta$ at time $t$. Then the surrogate regret of the set of distributions $P_t$ compared to the distribution $Q$ under the losses $g_t$ is defined by
\begin{align*}
S_T^Q & := \sum_{t=1}^T\Big(L(g_t,P_t)-\mathbb{E}_Q\big[g_t(\theta)\big]\Big) \\
 & = -\sum_{t=1}^T\ln\mathbb{E}_{P_t}\big[e^{-g_t(\theta)}\big]-\mathbb{E}_Q\Big[\sum_{t=1}^Tg_t(\theta)\Big].
\end{align*}
\end{definition}\noindent
For now, we set $\theta=(\eta,k)$ and $\Theta=[0,\frac{1}{2}]\times\{1,2,\dots,K\}$. Our loss is $g_t(\theta) = \hat{f}_t(\eta,k)$. For the surrogate regret under these losses we obtain the expression
\begin{equation}\label{eq:Sreg}
S_T^Q = -\sum_{t=1}^T\ln\mathbb{E}_{P_t}\big[e^{-\hat{f}_t(\eta,k)}\big] - \mathbb{E}_Q\Big[\sum_{t=1}^T\hat{f}_t(\eta,k)\Big].
\end{equation}
The surrogate regret now represents the difference between the total mix loss of the learner (who uses $P_t$ as its distributions) and the total expected loss of any other distribution $Q$. The goal of the new task is to keep the surrogate regret small. We can now reduce our original task to this new task by only keeping track of a probability distribution on the experts and marginalizing out the learning rate:
\[\mathbf{w}_t = \frac{\mathbb{E}_{P_t}[\eta\mathbf{e}_k]}{\mathbb{E}_{P_t}[\eta]}\]
To derive the Squint algorithm, we will take a look at the definition of the Exponential Weights algorithm (EW) \cite{ExpW}, which we will also use more often in the remainder of this thesis. 
\begin{definition}\label{def:EW}
Let $\Theta$ be a measurable space, let $\rho(\theta)$ be a prior distribution on $\Theta$ and let $g_t(\theta)\in\mathbb{R}$ be the loss of $\theta\in\Theta$ at time $t$. Then the Exponential Weights algorithm sets the densities of the probability distributions $P_t$ to be equal to
\[dP_{t+1}(\theta) = \frac{e^{-\eta_{EW}\sum_{s=1}^tg_s(\theta)}d\rho(\theta)}{Z_{t+1}}\]
where
\[Z_{t+1} = \int_{\Theta}e^{-\eta_{EW}\sum_{s=1}^tg_s(\theta)}d\rho(\theta)\]
is a normalization factor and $\eta_{EW}$ is a parameter for this algorithm.
\end{definition}\noindent
When we determine our $P_t$ according to this algorithm with prior $\gamma(\eta)\times\pi(k)$ and $\eta_{EW}=1$, the density of $P_{t+1}$ equals
\begin{equation}\label{eq:EW}
dP_{t+1}(\eta,k) = \frac{e^{-\sum_{s=1}^t\hat{f}_s(\eta,k)}d\gamma(\eta)\times\pi(k)}{Z_{t+1}}
\end{equation}
with
\[Z_{t+1} = \sum_{k=1}^K\int_0^{\frac{1}{2}}e^{-\sum_{s=1}^t\hat{f}_s(\eta,k)}d\gamma(\eta)\times\pi(k).\]
This now yields
\begin{align*}
\mathbf{w}_{t+1} &= \frac{\mathbb{E}_{P_{t+1}}[\eta\mathbf{e}_k]}{\mathbb{E}_{P_{t+1}}[\eta]} \\
 & = \frac{\mathbb{E}_{\gamma(\eta)\pi(k)}\big[e^{-\sum_{s=1}^t\hat{f}_s(\eta,k)}\eta\mathbf{e}_k\big]}{\mathbb{E}_{\gamma(\eta)\pi(k)}\big[e^{-\sum_{s=1}^t\hat{f}_s(\eta,k)}\eta\big]} \\ 
 & = \frac{\mathbb{E}_{\gamma(\eta)\pi(k)}\big[e^{\eta R_t^k-\eta^2V_t^k}\eta \mathbf{e}_k\big]}{\mathbb{E}_{\gamma(\eta)\pi(k)}\big[e^{\eta R_t^k-\eta^2V_t^k}\eta\big]},
\end{align*}
the Squint algorithm (\ref{eq:Squint}) for our original task. Hence, we have reduced Squint for the original learning task to a new algorithm for the surrogate task. We will now use the new task to bound the regret of Squint.

\subsection{Proving the Bound}
In this section we aim to prove the following theorem:
\begin{theorem}[{\cite{ExpW}}]
\label{thm:Rnochange}
Let $T\geq1$. Determine $P_t$ by using the Exponential Weights algorithm (see Definition \ref{def:EW}) with $\eta_{EW}=1$ and prior $\gamma(\eta)\times\pi(k)$, where $\gamma$ is the uniform distribution on $\Gamma=\big\{\frac{1}{2}\big\}\cup\big\{2^{-i}:i=1,2,\dots,\lceil ^2\log{\sqrt{T}}\rceil\big\}$ and $\pi$ is arbitrary. Then
\[R_T^{\mathcal{K}}\leq2\sqrt{2V_T^{\mathcal{K}}A_T^{\mathcal{K}}}+4A_T^{\mathcal{K}}\]
with
\[A_T^{\mathcal{K}} := \big(\ln\lceil ^2\log{\sqrt{T}}\rceil-\ln{\pi(\mathcal{K})}\big) \vee 1\]
where $x \vee y := \max\{x,y\}$.
\end{theorem}\noindent
Theorem \ref{thm:Rnochange} results in the bound
\[R_T^{\mathcal{K}} \preccurlyeq \sqrt{V_T^{\mathcal{K}}\big(\ln(\ln{T})-\ln{\pi(\mathcal{K})}\big)} + \ln(\ln{T})-\ln{\pi(\mathcal{K})}\]
for the mentioned choices of priors $\gamma(\eta)$ and $\pi(k)$. This is exactly the bound (\ref{eq:SquintRTk}) we stated in Section \ref{subsec:2.3}. To be able to prove Theorem \ref{thm:Rnochange}, we need the following lemma:
\begin{lemma} [{\cite[Theorem 8]{ExpW}}]
\label{lem:surreg}
Use any algorithm to determine $P_t$. Then for every probability distribution $Q(\eta,k)$:
\[\mathbb{E}_Q[\eta R_T^k] \leq \mathbb{E}_Q[\eta^2 V_T^k] + S_T^Q.\]
\end{lemma}
\noindent
We will give the proof of this lemma at the end of this section. First, we want to show how helpful it is by using it for proving Theorem \ref{thm:Rnochange}.
\begin{proof}[\textbf{Proof of Theorem \ref{thm:Rnochange}}]
Since the instantaneous regret $r_t^k$ is bounded by $1$, the statement is trivial for $T=1$. So we now assume that $T\geq2$, such that the expression of $\Gamma$ simplifies to $\Gamma=\big\{2^{-i}:i=1,2,\dots,\lceil ^2\log{\sqrt{T}}\rceil\big\}$ and thus $|\Gamma|=\lceil ^2\log{\sqrt{T}}\rceil$. We will use this property later on in the proof. Now take $Q = \delta_{\ddot{\eta}}\times\pi(\cdot|\mathcal{K})$ for some $\ddot{\eta}\in\Gamma$, $\delta$ the Dirac delta function and $\pi(\cdot|\mathcal{K})$ the prior $\pi(k)$ conditioned on $\mathcal{K}$. Then Lemma \ref{lem:surreg} yields
\[R_T^{\mathcal{K}} \leq \ddot{\eta}V_T^{\mathcal{K}}+\frac{1}{\ddot{\eta}}S_T^Q.\]
The surrogate regret is bounded by $S_T^Q \leq \text{KL}(Q||\gamma\times\pi)$ where $\text{KL}(\tilde{Q}||\tilde{P}) := \mathbb{E}_{\tilde{Q}}\ln\big(\frac{d\tilde{Q}}{d\tilde{P}}\big)$ is the Kullback-Leibler divergence of $\tilde{Q}$ from $\tilde{P}$ for any probability distributions $\tilde{Q}$ and $\tilde{P}$ on the same measurable space. This result is obtained by applying Lemma \ref{lem:Sreg}. This lemma is a general result, which can be used for any measurable loss function on any measurable space.
\begin{lemma}\label{lem:Sreg}
Let $\Theta$ be a measurable space and let $Q$ and $P_t$ be probability distributions on $\Theta$ where $Q$ is arbitrary and $P_t$ is determined according to the Exponential Weights algorithm with $\eta_{EW}=1$, a prior $\rho(\theta)$ and measurable losses $g_t(\theta)\in\mathbb{R}$ for $\theta\in\Theta$. Then the surrogate regret is bounded by
\begin{align*}
S_T^Q & \leq \text{\normalfont{KL}}(Q||\rho).
\end{align*}
\end{lemma}\noindent
For the proof of this lemma, we will use Lemma \ref{lem:DoVa}, which is originally from Donsker and Varadhan. The proof of Lemma \ref{lem:DoVa} is given in \cite{CompLem}.
\begin{lemma}[{Donsker-Varadhan Lemma, \cite[Lemma 1]{CompLem}}]\label{lem:DoVa}
For any measurable function $\phi:\Theta\rightarrow\mathbb{R}$ and any distributions $P$ and $Q$ on $\Theta$, we have
\[\mathbb{E}_Q[\phi(\theta)]-\ln\mathbb{E}_P[e^{\phi(\theta)}]\leq \text{\normalfont{KL}}(Q||P).\]
\end{lemma}
\begin{proof}[\textbf{Proof of Lemma \ref{lem:Sreg}}]
\begin{align*}
S_T^Q & = -\sum_{t=1}^T\ln\mathbb{E}_{P_t}\big[e^{-g_t(\theta)}\big] - \mathbb{E}_Q\Big[\sum_{t=1}^Tg_t(\theta)\Big] \\
 & = -\sum_{t=1}^T\ln\Big(\frac{Z_{t+1}}{Z_t}\Big) - \mathbb{E}_Q\Big[\sum_{t=1}^Tg_t(\theta)\Big] \\
 & = -\ln(Z_{T+1})+\ln(Z_1) - \mathbb{E}_Q\Big[\sum_{t=1}^Tg_t(\theta)\Big] \\
 & = -\ln(Z_{T+1}) - \mathbb{E}_Q\Big[\sum_{t=1}^Tg_t(\theta)\Big] \\
 & = -\ln\mathbb{E}_{\rho(\theta)}\big[e^{-\sum_{t=1}^Tg_t(\theta)}\big] - \mathbb{E}_Q\Big[\sum_{t=1}^Tg_t(\theta)\Big]
\end{align*}
We now apply Lemma \ref{lem:DoVa} and substitute $\phi(\theta) = -\sum_{t=1}^Tg_t(\theta)$ and $P=\rho$ to obtain the result
\[S_T^Q \leq \text{KL}(Q||\rho).\]
\end{proof}\noindent
This is a general result for any measurable loss function $g_t(\theta)$ on any measurable space $\Theta$ with prior $\rho(\theta)$. In the next chapter, we will use this result again. For now, we will apply it to the proof of Theorem \ref{thm:Rnochange} by substituting $\Theta = [0,\frac{1}{2}]\times\{1,2,\dots,K\}$, $\rho=\gamma\times\pi$ and $g_t(\theta) = \hat{f}_t(\eta,k)$. This yields the desired inequality:
\[S_T^Q \leq \text{KL}(Q||\gamma\times\pi)\]
Remember that $\pi$ is arbitrary, $\gamma$ is the uniform distribution on $\Gamma$ and $Q = \delta_{\ddot{\eta}}\times\pi(\cdot|\mathcal{K})$. Since we assumed that $T\geq2$, we have $|\Gamma|=\lceil ^2\log{\sqrt{T}}\rceil$. Now we can write out the Kullback-Leibler divergence of $Q$ from $\gamma\times\pi$ as follows:
\begin{align*}
\text{KL}(Q||\gamma\times\pi) & = \mathbb{E}_Q\ln\Big(\frac{dQ}{d(\gamma\times\pi)}\Big) \\
 & = \sum_{k\in\mathcal{K}}Q(\ddot{\eta},k)\ln\Big(\frac{Q(\ddot{\eta},k)}{\gamma({\ddot{\eta}})\cdot\pi(k)}\Big) \\
 & = \sum_{k\in\mathcal{K}}\pi(k|\mathcal{K})\ln\Big(\frac{\lceil ^2\log{\sqrt{T}}\rceil\cdot\pi(k|\mathcal{K})}{\pi(k)}\Big) \\
 & = \sum_{k\in\mathcal{K}}\frac{\pi(k)}{\pi(\mathcal{K})}\ln\Big(\frac{\lceil ^2\log{\sqrt{T}}\rceil\cdot\pi(k)}{\pi(k)\cdot\pi(\mathcal{K})}\Big) \\
 & = \ln\lceil ^2\log{\sqrt{T}}\rceil-\ln\pi(\mathcal{K}) \\
 & \leq A_T^{\mathcal{K}}.
\end{align*}
Combining these inequalities gives
\[S_T^Q \leq \text{KL}(Q||\gamma\times\pi) \leq A_T^{\mathcal{K}}.\]
Implementing this in the initial inequality yields
\[R_T^{\mathcal{K}} \leq \ddot{\eta}V_T^{\mathcal{K}}+\frac{1}{\ddot{\eta}}A_T^{\mathcal{K}}.\]
We would like to minimize the bound over $\ddot{\eta}$, such that we have the best possible bound. Unfortunately though, $\ddot{\eta}$ is an element of the discrete set $\Gamma$ and we want to perform continuous optimization. However, by the definition of $\Gamma$ we conclude that for every $\hat{\eta}\in\big[\frac{1}{2\sqrt{T}},\frac{1}{2}\big]$ there exists a $\ddot{\eta}\in\Gamma$ which is within a factor 2 of $\hat{\eta}$ (i.e.\ $\hat{\eta}\leq\ddot{\eta}\leq2\hat{\eta}$). This gives the inequality
\[R_T^{\mathcal{K}} \leq 2\hat{\eta}V_T^{\mathcal{K}}+\frac{1}{\hat{\eta}}A_T^{\mathcal{K}}.\]
Continuous optimization of this bound using the first and second derivative then yields a minimum for
\[\hat{\eta}_{min} = \sqrt{\frac{A_T^{\mathcal{K}}}{2V_T^{\mathcal{K}}}}.\]
We still have to check whether this minimum is reached on the interval $\big[\frac{1}{2\sqrt{T}},\frac{1}{2}\big]$, so we distinguish three different possibilities:
\begin{enumerate}
\item $\hat{\eta}_{min}<\frac{1}{2\sqrt{T}}$: Substitution gives $\frac{A_T^{\mathcal{K}}}{2V_T^{\mathcal{K}}}<\frac{1}{4T}$. Since the minimum of $A_T^{\mathcal{K}}$ is 1 and $V_T^{\mathcal{K}}$ does not exceed $T$, we find $\frac{1}{2T}\leq\frac{A_T^{\mathcal{K}}}{2V_T^{\mathcal{K}}}<\frac{1}{4T}$, which is impossible. Hence, this case can be excluded.
\item $\frac{1}{2\sqrt{T}}\leq\hat{\eta}_{min}<\frac{1}{2}$: Since $\hat{\eta}_{min}$ is an element of the interval, we can substitute it in the inequality:
\begin{align*}
R_T^{\mathcal{K}} & \leq 2\hat{\eta}_{min}V_T^{\mathcal{K}} + \frac{1}{\hat{\eta}_{min}}A_T^{\mathcal{K}} \\
 & = 2\sqrt{\frac{A_T^{\mathcal{K}}}{2V_T^{\mathcal{K}}}}\cdot V_T^{\mathcal{K}} + \sqrt{\frac{2V_T^{\mathcal{K}}}{A_T^{\mathcal{K}}}}\cdot A_T^{\mathcal{K}} \\
 & = 2\sqrt{2V_T^{\mathcal{K}}A_T^{\mathcal{K}}}
\end{align*}
\item $\hat{\eta}_{min}\geq\frac{1}{2}$: Substitution gives $\frac{A_T^{\mathcal{K}}}{2V_T^{\mathcal{K}}} \geq \frac{1}{4}$ and thus $V_T^{\mathcal{K}}\leq2A_T^{\mathcal{K}}$. Since the first derivative of the bound $2\hat{\eta}V_T^{\mathcal{K}}+\frac{1}{\hat{\eta}}A_T^{\mathcal{K}}$ now is non-positive for all $\hat{\eta}\in(0,\frac{1}{2}]$, we conclude that the minimum of the bound on $\big[\frac{1}{2\sqrt{T}},\frac{1}{2}\big]$ is obtained for $\hat{\eta}_{min}=\frac{1}{2}$. This gives
\begin{align*}
R_T^{\mathcal{K}} & \leq 2\hat{\eta}_{min}V_T^{\mathcal{K}} + \frac{1}{\hat{\eta}_{min}}A_T^{\mathcal{K}} \\
 & = V_T^{\mathcal{K}} + 2A_T^{\mathcal{K}} \\
 & \leq 4A_T^{\mathcal{K}}.
\end{align*}
\end{enumerate}
So in all cases we find that
\[R_T^{\mathcal{K}} \leq 2\sqrt{2V_T^{\mathcal{K}}A_T^{\mathcal{K}}} + 4A_T^{\mathcal{K}}\]
which is the initial statement, we wanted to prove.
\end{proof}
\noindent
Now, to complete the proof of Theorem \ref{thm:Rnochange}, we will give the proof of Lemma \ref{lem:surreg}.
\begin{proof}[\textbf{Proof of Lemma \ref{lem:surreg}}]
Define $m_t := L(\hat{f}_t,P_t) = -\ln{\mathbb{E}_{P_t}\big[e^{-\hat{f}_t(\eta,k)}\big]}$ for every $t\in\{1,2,\dots,T\}$. Then
\begin{align*}
e^{-m_t} & = \mathbb{E}_{P_t}\big[e^{-\hat{f}_t(\eta,k)}\big] \\
 & = \mathbb{E}_{P_t}\big[e^{\eta r_t^k-\eta^2(r_t^k)^2}\big].
\end{align*}
Since $e^{x-x^2}\leq1+x$ for $x\geq-\frac{1}{2}$ \cite[Lemma 1]{Ineq} and $\eta r_t^k\in\big[-\frac{1}{2},\frac{1}{2}\big]$, we obtain
\begin{align*}
e^{-m_t} & \leq \mathbb{E}_{P_t}[1+\eta r_t^k] \\
 & = 1 + \mathbb{E}_{P_t}[\eta]\cdot\mathbf{w}_t^{\top}\mathbf{l}_t-\mathbb{E}_{P_t}[\eta\cdot l_t^k] \\
 & = 1 + \mathbb{E}_{P_t}[\eta]\cdot\Big(\frac{\mathbb{E}_{P_t}[\eta\mathbf{e}_k]}{\mathbb{E}_{P_t}[\eta]}\Big)^{\top}\mathbf{l}_t-\mathbb{E}_{P_t}[\eta\cdot l_t^k] \\
 & = 1 + \mathbb{E}_{P_t}[\eta\cdot l_t^k] - \mathbb{E}_{P_t}[\eta\cdot l_t^k] \\
 & = 1.
\end{align*}
From this, we conclude that $m_t\geq0$ for every $t$. Using (\ref{eq:Sreg}), the relevant expression of $S_T^Q$, we find
\begin{align*}
0 & \leq \sum_{t=1}^Tm_t \\
 & = -\sum_{t=1}^T \ln{\mathbb{E}_{P_t}\big[e^{-\hat{f}_t(\eta,k)}\big]} \\
 & = S_T^Q + \mathbb{E}_Q\big[\sum_{t=1}^T\hat{f}_t(\eta,k)\big] \\
 & = S_T^Q + \mathbb{E}_Q[-\eta R_T^k + \eta^2 V_T^k].
\end{align*}
Hence, we obtain the inequality
\[\mathbb{E}_Q[\eta R_T^k] \leq \mathbb{E}_Q[\eta^2V_T^k]+S_T^Q.\]
\end{proof}
\noindent
In conclusion, we have proved Theorem \ref{thm:Rnochange} and thus the bound on Squint. In the next chapter we will present an adjustment of Squint such that it will function well in a changing environment. We will combine techniques from this chapter with ingredients from the previous chapter to find a new bound.

\newpage
\section{Squint in a Changing Environment}\label{sec:5}
We present our own algorithm, named Squint-CE, which is based on Squint, but now yields a favorable bound on the regret in a changing environment. The pseudocode for Squint-CE is given by Algorithm \ref{alg:SCE}. Theorem \ref{thm:Rchange} provides the corresponding bound. In this theorem we use the following aforementioned definitions: $R_I^{\mathcal{K}}:=\mathbb{E}_{\pi(k|\mathcal{K})}R_I^k$ with $R_I^k:=\sum_{t=I_1}^{I_2}r_t^k$ and $V_I^{\mathcal{K}}:=\mathbb{E}_{\pi(k|\mathcal{K})}V_I^k$ with $V_I^k:=\sum_{t=I_1}^{I_2}(r_t^k)^2$.
\begin{theorem}
\label{thm:Rchange}
Let $\mathcal{B} = \{b_J:J=[J_1,J_2]\in\mathcal{J}\text{ with }J_2\leq T\}$ be the set of black-box algorithms operating solely on a geometric covering interval with endpoint not larger than $T$. Furthermore, let $\tau(b)$ be the uniform distribution on $\mathcal{B}$, let $\gamma(\eta)$ be the uniform distribution on $\Gamma=\big\{\frac{1}{2}\big\}\cup\big\{2^{-i}:i=1,2,\dots,\lceil ^2\log{\sqrt{T}}\rceil\big\}$ and let $\pi(k)$ be arbitrary. Then for any contiguous interval $I=[I_1,I_2]\subseteq\{1,2,\dots,T\}$, Algorithm \ref{alg:SCE} yields
\[R_I^{\mathcal{K}}\leq2\sqrt{2V_I^{\mathcal{K}}A_I^{\mathcal{K}}}+4A_I^{\mathcal{K}}\]
with
\[A_I^{\mathcal{K}} := \Big(2\cdot {}^2\log\big(|I|+2\big) \cdot \big(\ln(2T)+\ln\lceil^2\log\sqrt{T}\rceil-\ln\pi(\mathcal{K})\big)\Big) \vee 1\]
where $x \vee y := \max\{x,y\}$.
\end{theorem}\noindent
Theorem \ref{thm:Rchange} results in the bound
\begin{equation}\label{eq:RIKbbound}
\begin{aligned}
R_I^{\mathcal{K}} & \preccurlyeq \sqrt{\ln|I|\cdot V_I^{\mathcal{K}}\big(\ln T + \ln(\ln T)-\ln\pi(\mathcal{K})\big)} \\
 & \qquad + \ln|I|\cdot \big(\ln T + \ln(\ln T)-\ln\pi(\mathcal{K})\big).
\end{aligned}
\end{equation}\noindent
This is an improvement on the previous bound
\begin{equation}\tag{\ref{eq:RIKwbound}}
\begin{aligned}
R_I^{\mathcal{K}} & \preccurlyeq \sqrt{|I|\ln{I_2}} + \sqrt{\ln|I|\cdot V_I^{\mathcal{K}}\big(\ln(\ln{|I|})-\ln\pi(\mathcal{K})\big)}\\
& \qquad + \ln|I|\cdot\big(\ln(\ln|I|)-\ln\pi(\mathcal{K})\big)
\end{aligned}
\end{equation}\noindent
which was obtained by applying CBCE to Squint in Chapter \ref{sec:3}. Our new bound (\ref{eq:RIKbbound}) does not contain the dominating term $\sqrt{|I|\ln I_2}$, which made the advantage Squint had over Hedge vanish in a changing environment. The only cost we now pay is that we use $\ln(\ln T)$ instead of $\ln(\ln|I|)$ and a term $\ln T$ is added. Fortunately though, the difference between $\ln(\ln T)$ and $\ln(\ln|I|)$ is minimal as they both practically behave like small constants and can be neglected. Moreover, the term $\ln T$ is a price we are willing to pay, since it grows relatively slow.\par
In this chapter we will show how we created Squint-CE and prove Theorem \ref{thm:Rchange}. In Section \ref{subsec:newtau}, the final section of this chapter, we will present a different prior $\tau(b)$ which yields a slightly better bound.
\begin{algorithm}[!ht]
\caption{Squint-CE}\label{alg:SCE}
\begin{algorithmic}
\State \textbf{Input:} A set of black-box algorithms $\mathcal{B}$, a set of all active black-box algorithms $\mathcal{B}_t\subseteq\mathcal{B}$ for each round $t=1,2,\dots,T$, a prior distribution $\tau(b)$ on $\mathcal{B}$, a prior distribution $\gamma(\eta)$ on a discrete set $\Gamma\subset[0,\frac{1}{2}]$ and a prior distribution $\pi(k)$ on $\{1,2,\dots,K\}$.
\For{$b\in\mathcal{B}$}
	\State $G_0^b \gets 0$
	\State $R_0^k(b) \gets 0$ \textbf{for} $k\in\{1,2\dots,K\}$
	\State $V_0^k(b) \gets 0$ \textbf{for} $k\in\{1,2\dots,K\}$
	\State $F_0^b(\eta,k) \gets 0$ \textbf{for} $\eta\in\Gamma$ and $k\in\{1,2\dots,K\}$
\EndFor
\For{$t=1$ \textbf{to} $T$}
	\For{$\eta\in\Gamma$ and $k\in\{1,2,\dots,K\}$}
		\State $P_t^b(\eta,k) \gets \frac{e^{-F_{t-1}^b(\eta,k)}\gamma(\eta)\times\pi(k)}{\sum_{k=1}^K\sum_{\eta\in\Gamma}e^{-F_{t-1}^b(\eta,k)}\gamma(\eta)\times\pi(k)}$ \textbf{for} $b\in\mathcal{B}_t$
	\EndFor
	\State $\tilde{q}_t^b \gets \frac{e^{-G_{t-1}^b}\tau(b)}{\sum_{b'\in\mathcal{B}}e^{-G_{t-1}^{b'}}\tau(b')}$ \textbf{for} $b\in\mathcal{B}$
	\State $q_t^b \gets \left\{\begin{matrix} \frac{1}{\tilde{q}_{t}(\mathcal{B}_{t})}\cdot \tilde{q}_{t}^b & \textbf{for } b\in\mathcal{B}_{t} \\ 0 & \textbf{for } b\in\mathcal{B}\setminus\mathcal{B}_{t}\end{matrix}\right.$
	\State $\tau_{t}(b) \gets \left\{\begin{matrix} \frac{1}{\tau(\mathcal{B}_{t})}\cdot\tau(b)&& \textbf{for }b\in\mathcal{B}_{t} \\ 0 && \textbf{for } b\in\mathcal{B}\setminus\mathcal{B}_{t}\end{matrix}\right.$
	\State $\mathbf{w}_{t} \gets \frac{\mathbb{E}_{\tau_{t}(b)\gamma(\eta)\pi(k)}\big[e^{-G_{t-1}^b + \eta R_{t-1}^k(b)-\eta^2V_{t-1}^k(b)}\eta\mathbf{e}_k\big]}{\mathbb{E}_{\tau_{t}(b)\gamma(\eta)\pi(k)}\big[e^{-G_{t-1}^b + \eta R_{t-1}^k(b)-\eta^2V_{t-1}^k(b)}\eta\big]}$
	\State Obtain the losses $\mathbf{l}_{t}$ from the environment
	\For{$k\in\{1,2,\dots,K\}$}
		\State $r_{t}^k \gets \mathbf{w}_{t}^{\top}\mathbf{l}_t-l_t^k$
		\State $R_t^k(b) \gets \left\{\begin{matrix} R_{t-1}^k + r_t^k && \textbf{for } b\in\mathcal{B}_{t} \\ R_{t-1}^k && \textbf{for } b\in\mathcal{B}\setminus\mathcal{B}_{t}\end{matrix}\right.$
		\State $V_t^k(b) \gets \left\{\begin{matrix} V_{t-1}^k + (r_t^k)^2 && \textbf{for } b\in\mathcal{B}_{t} \\ V_{t-1}^k && \textbf{for } b\in\mathcal{B}\setminus\mathcal{B}_{t}\end{matrix}\right.$
		\For{$\eta\in\Gamma$}
			\State $\hat{f}_t(\eta,k) \gets -\eta r_t^k + \eta^2(r_t^k)^2$
			\State $F_t^b(\eta,k) \gets \left\{\begin{matrix} F_{t-1}^b(\eta,k) + \hat{f}_t(\eta,k) && \textbf{for } b\in\mathcal{B}_{t} \\ F_{t-1}^b(\eta,k) && \textbf{for } b\in\mathcal{B}\setminus\mathcal{B}_{t}\end{matrix}\right.$
		\EndFor
	\EndFor
	\State $g_t(b) \gets -\ln\mathbb{E}_{P_t^b}\big[e^{-\hat{f}_t(\eta,k)}\big]$ \textbf{for} $b\in\mathcal{B}_t$
	\State $\hat{g}_t \gets -\ln\mathbb{E}_{\mathbf{q}_t}\big[e^{-g_t(b)}\big]$
	\State $g_t(b) \gets \hat{g}_t$ \textbf{for} $b\in\mathcal{B}\setminus\mathcal{B}_{t}$
	\State $G_t^b \gets G_{t-1}^b+g_t(b)$ \textbf{for} $b\in\mathcal{B}$
\EndFor
\end{algorithmic}
\end{algorithm}

\subsection{Surrogate Task in a Changing Environment}
We again consider the surrogate task, introduced in Chapter \ref{sec:4}. So the aim is to learn the best learning rate $\eta\in[0,\frac{1}{2}]$ and the best expert $k\in\{1,2,\dots,K\}$. However, this time we are in a changing environment, which means that the best learning rate and the best expert can change over time. The learner will now use black-box algorithms and a meta algorithm as introduced in Chapter \ref{sec:3} to determine its probability distribution $P_t^{\mathcal{M}(\mathcal{B})}(\eta,k)$ on the learning rates and experts for each time step. We then again marginalize out the learning rate to obtain the weight vector for the original task:
\[\mathbf{w}_t=\frac{\mathbb{E}_{P_t^{\mathcal{M}(\mathcal{B})}}[\eta\mathbf{e}_k]}{\mathbb{E}_{P_t^{\mathcal{M}(\mathcal{B})}}[\eta]}\]
The black-box algorithms, used by the learner, will all run on different intervals and thus use only the data available on these intervals. They keep track of their own probability distributions $P_t^b(\eta,k)$ on the learning rates and experts. To determine these distributions they use Exponential Weights (\ref{eq:EW}) with $\eta_{EW}=1$ and prior $\gamma(\eta)\times\pi(k)$, just like in Chapter \ref{sec:4}. We again denote the set of used black-box algorithms by $\mathcal{B}$ with $\mathcal{B}_t\subseteq\mathcal{B}$ the set of active black-box algorithms at time $t$.\par
The meta algorithm will keep track of a probability distribution $\mathbf{q}_t$ on the black-boxes. To determine this distribution it will use Exponential Weights (see Definition \ref{def:EW}) on the black-boxes, again with $\eta_{EW}=1$ and with prior $\tau(b)$, after which it conditions on the set of active black-box algorithms, since it cannot pick any inactive $b$'s. The meta algorithm needs losses of the black-box algorithms in order to apply Exponential Weights. We denote these losses by $g_t(b)$ and set them equal to the mix losses $L(\hat{f}_t,P_t^b)$ where $\hat{f}_t(\eta,k) := -\eta r_t^k + \eta^2(r_t^k)^2$ is the surrogate loss obtained by the meta algorithm in combination with the black-box algorithms. When a black-box algorithm is not active at time $t$, it does not output a distribution $P_t^b$, so then we set $g_t(b)$ to be equal to some constant $c_t$, which only depends on time and not on $b$. We will define this constant later. So we have
\[g_t(b) := \left\{\begin{matrix}-\ln\mathbb{E}_{P_t^b}\big[e^{-\hat{f}_t(\eta,k)}\big] & \text{if } b\in\mathcal{B}_t \\
c_t & \text{if } b\notin\mathcal{B}_t\end{matrix}\right..\]
As mentioned before, to determine the probability vector $\mathbf{q}_t$ with components $q_t^b$ on the set $\mathcal{B}$, we first set up a probability vector $\tilde{\mathbf{q}}_t$ using Exponential Weights. This vector then has components
\[\tilde{q}_{t+1}^b = \frac{1}{Y_{t+1}}e^{-\sum_{s=1}^tg_s(b)}\tau(b)\]
where
\[Y_{t+1} = \sum_{b'\in\mathcal{B}}e^{-\sum_{s=1}^tg_s(b')}\tau(b')\]
is a normalization factor. Then we condition it on the active black-box algorithms to obtain $\mathbf{q}_t$:
\[q_{t}^b = \left\{\begin{matrix} \frac{1}{\tilde{q}_{t}(\mathcal{B}_{t})}\cdot \tilde{q}_{t}^b & \text{if } b\in\mathcal{B}_{t} \\ 0 & \text{if } b\notin\mathcal{B}_{t}\end{matrix}\right.\]
where $\tilde{q}_{t}(\mathcal{B}_{t})$ is the weight $\tilde{\mathbf{q}}_t$ puts on the set $\mathcal{B}_t$. Using this probability vector and the distributions $P_t^b$ of the black-box algorithms, the learner then plays the probability distribution
\[P_t^{\mathcal{M}(\mathcal{B})} = \mathbb{E}_{\mathbf{q}_t}\big[P_t^b\big]\]
on the learning rates and experts for each time step $t$. So then our new algorithm for the original task, Squint-CE, determines $\mathbf{w}_{t+1}$ by
\begin{align*}
\mathbf{w}_{t+1} & = \frac{\mathbb{E}_{P_{t+1}^{\mathcal{M}(\mathcal{B})}}[\eta\mathbf{e}_k]}{\mathbb{E}_{P_{t+1}^{\mathcal{M}(\mathcal{B})}}[\eta]} \\
 & = \frac{\mathbb{E}_{\mathbf{q}_{t+1}}\big[\mathbb{E}_{P_{t+1}^{b}}[\eta\mathbf{e}_k]\big]}{\mathbb{E}_{\mathbf{q}_{t+1}}\big[\mathbb{E}_{P_{t+1}^{b}}[\eta]\big]} \\
 & = \frac{\mathbb{E}_{\mathbf{q}_{t+1}}\Big[\mathbb{E}_{\gamma(\eta)\pi(k)}\big[e^{-\sum_{s=J_1^b}^t\hat{f}_s(\eta,k)}\eta\mathbf{e}_k\big]\Big]}{\mathbb{E}_{\mathbf{q}_{t+1}}\Big[\mathbb{E}_{\gamma(\eta)\pi(k)}\big[e^{-\sum_{s=J_1^b}^t\hat{f}_s(\eta,k)}\eta\big]\Big]} \\
 & = \frac{\mathbb{E}_{\mathbf{q}_{t+1}}\Big[\mathbb{E}_{\gamma(\eta)\pi(k)}\big[e^{\eta R_{[J_1^b,t]}^k-\eta^2V_{[J_1^b,t]}^k}\eta\mathbf{e}_k\big]\Big]}{\mathbb{E}_{\mathbf{q}_{t+1}}\Big[\mathbb{E}_{\gamma(\eta)\pi(k)}\big[e^{\eta R_{[J_1^b,t]}^k-\eta^2V_{[J_1^b,t]}^k}\eta\big]\Big]} \\
 & = \frac{\mathbb{E}_{\tau_{t+1}(b)\gamma(\eta)\pi(k)}\Big[e^{-\sum_{s=1}^t g_s(b) + \eta R_{[J_1^b,t]}^k-\eta^2V_{[J_1^b,t]}^k}\eta\mathbf{e}_k\Big]}{\mathbb{E}_{\tau_{t+1}(b)\gamma(\eta)\pi(k)}\Big[e^{-\sum_{s=1}^t g_s(b) + \eta R_{[J_1^b,t]}^k-\eta^2V_{[J_1^b,t]}^k}\eta\Big]} \\
 & = \frac{\mathbb{E}_{\tau_{t+1}(b)\gamma(\eta)\pi(k)}\Big[e^{-G_t^{b} + \eta R_{[J_1^b,t]}^k-\eta^2V_{[J_1^b,t]}^k}\eta\mathbf{e}_k\Big]}{\mathbb{E}_{\tau_{t+1}(b)\gamma(\eta)\pi(k)}\Big[e^{-G_t^{b} + \eta R_{[J_1^b,t]}^k-\eta^2V_{[J_1^b,t]}^k}\eta\Big]}
\end{align*}
where $J_1^b$ is the starting point of the interval on which black-box $b$ is active, $G_T^b:=\sum_{t=1}^T g_t(b)$ is the total loss of black-box $b$ up until time $T$ and
\[\tau_t(b) = \left\{\begin{matrix}\frac{1}{\tau(\mathcal{B}_t)}\cdot\tau(b) & \text{if } b\in\mathcal{B}_t \\
0 & \text{if } b\notin\mathcal{B}_t\end{matrix}\right.\]
is the prior $\tau(b)$ conditioned on the active black-box algorithms at time $t$. Now that we know how the algorithm computes the weight vector for the original task, we will work towards the bound on the regret. For this, we will first look into the loss of the learner.

\subsection{Loss of the Learner}
We define the loss of the learner $\hat{g}_t$ at time $t$ to be the mix loss of $g_t(b)$ under the distribution $\mathbf{q}_t$:
\[\hat{g}_t := L(g_t, \mathbf{q}_t) = -\ln\mathbb{E}_{\mathbf{q}_t}\big[e^{-g_t(b)}\big]\]
In order to make the analysis of the regret easier, we will now present an equivalence of the learner's loss, as similarly done by Adamskiy et al.\ \cite{EquivLoss}. For this, we use the unconditioned probability vector $\tilde{\mathbf{q}}_t$. Moreover, we now define
\[c_t := \hat{g}_t.\]
So if $b$ is active, its loss $g_t(b)$ will be equal to the mix loss we had before, but when $b$ is not active, its loss is equal to the learner's loss. Since $\hat{g}_t$ is determined using only the losses of active black-boxes, this is a well defined choice for $g_t(b)$. Now the loss of the learner can be rewritten as
\[\hat{g}_t = -\ln\mathbb{E}_{\tilde{\mathbf{q}}_t}\big[e^{-g_t(b)}\big].\]
We will prove this equality directly. Note that for all $b\in\mathcal{B}_t$ we have $\tilde{q}_t^b = \tilde{q}_t(\mathcal{B}_t)\cdot q_t^b$. Then
\begin{align*}
-\ln\mathbb{E}_{\tilde{\mathbf{q}}_t}\big[e^{-g_t(b)}\big] & = -\ln\Big(\tilde{q}_t(\mathcal{B}_t)\cdot\mathbb{E}_{\mathbf{q}_t}\big[e^{-g_t(b)}\big] + \big(1-\tilde{q}_t(\mathcal{B}_t)\big)\cdot e^{-\hat{g}_t}\Big) \\
 & = -\ln\Big(\tilde{q}_t(\mathcal{B}_t)\cdot\mathbb{E}_{\mathbf{q}_t}\big[e^{-g_t(b)}\big] + \big(1-\tilde{q}_t(\mathcal{B}_t)\big)\cdot e^{\ln\mathbb{E}_{\mathbf{q}_t}[e^{-g_t(b)}]}\Big) \\
 & = -\ln\Big(\tilde{q}_t(\mathcal{B}_t)\cdot\mathbb{E}_{\mathbf{q}_t}\big[e^{-g_t(b)}\big] + \big(1-\tilde{q}_t(\mathcal{B}_t)\big)\cdot \mathbb{E}_{\mathbf{q}_t}\big[e^{-g_t(b)}\big]\Big) \\
 & = -\ln\mathbb{E}_{\mathbf{q}_t}\big[e^{-g_t(b)}\big] \\
 & = \hat{g}_t.
\end{align*}
We conclude that the loss of the learner, $\hat{g}_t$, is the same under the two probability distributions $\mathbf{q}_t$ and $\tilde{\mathbf{q}}_t$. Of course, when we run the algorithm, we use $\mathbf{q}_t$ to randomly pick a black-box algorithm, since an algorithm using $\tilde{\mathbf{q}}_t$ can also pick an inactive black-box algorithm that does not compute a probability vector. But for the analysis of the algorithm we prefer to use $\tilde{\mathbf{q}}_t$, as this vector is computed using the regular Exponential Weights algorithm without conditioning, which enables us to use results from Chapter \ref{sec:4}.

\subsection{Surrogate Regret Analysis}\label{subsec:SRAgt}
Our goal is to give a similar proof as in Chapter \ref{sec:4} to bound the regular regret $R_I^{\mathcal{K}}$ on any contiguous interval $I$. We will therefore present a bound on the surrogate regret of the learner compared to the probability distribution $Q$ on $(\eta,k)$, after which we use an adjusted version of Lemma \ref{lem:surreg} to then complete the proof.\par
Consider an interval $I=[I_1,I_2]\subseteq\{1,2,\dots,T\}$. Recall from Chapter \ref{sec:3} that this interval can be partitioned by a set $\{J^{(i)}\in\mathcal{J}:-c\leq i \leq d\}$ of consecutive and disjoint geometric covering intervals, which successively double and then successively halve in length. We used black-box algorithms $b_{J^{(i)}}$ which solely operated on those intervals $J^{(i)}=[J_1^{(i)},J_2^{(i)}]$ and not outside of them to decompose the regret. We will do the same for the surrogate regret of the learner. First, we define this surrogate regret similar to Definition \ref{def:surreg}.
\[S_I^Q\big(\mathcal{M}(\mathcal{B})\big) := \sum_{t=I_1}^{I_2}\Big(L\big(\hat{f}_t,P_t^{\mathcal{M}(\mathcal{B})}\big)-\mathbb{E}_Q\big[\hat{f}_t(\eta,k)\big]\Big)\]
The first term in the summation is actually equal to the loss of the learner:
\begin{align*}
\hat{g}_t & = -\ln\mathbb{E}_{\mathbf{q}_t}\big[e^{-g_t(b)}\big] \\
 & = -\ln\mathbb{E}_{\mathbf{q}_t}\big[e^{\ln\mathbb{E}_{P_t^b}[e^{-\hat{f}_t(\eta,k)}]}\big] \\
 & = -\ln\mathbb{E}_{\mathbf{q}_t}\big[\mathbb{E}_{P_t^b}[e^{-\hat{f}_t(\eta,k)}]\big] \\
 & = -\ln\mathbb{E}_{P_t^{\mathcal{M}(\mathcal{B})}}\big[e^{-\hat{f}_t(\eta,k)}\big] \\
 & = L\big(\hat{f}_t,P_t^{\mathcal{M}(\mathcal{B})}\big)
\end{align*}
Using this equality, we now decompose the surrogate regret.
\begin{align*}
S_I^Q\big(\mathcal{M}(\mathcal{B})\big) & = \sum_{t=I_1}^{I_2}\Big(\hat{g}_t-\mathbb{E}_Q[\hat{f}_t(\eta,k)]\Big) \\
 & = \sum_{i=-c}^d\sum_{t\in J^{(i)}}\Big(\hat{g}_t-g_t(b_{J^{(i)}}) + g_t(b_{J^{(i)}}) - \mathbb{E}_Q[\hat{f}_t(\eta,k)]\Big) \\
 & = \sum_{i=-c}^d\bigg(\sum_{t\in J^{(i)}}\Big(\hat{g}_t-g_t(b_{J^{(i)}})\Big) + \sum_{t\in J^{(i)}}\Big(g_t(b_{J^{(i)}}) - \mathbb{E}_Q[\hat{f}_t(\eta,k)]\Big)\bigg) \\
 & = \sum_{i=-c}^d\Big( S_{J^{(i)}}^{b_{J^{(i)}}}(\mathcal{M}) + S_{J^{(i)}}^Q(b_{J^{(i)}})\Big)
\end{align*}
Here $S_{J^{(i)}}^{b_{J^{(i)}}}(\mathcal{M})$ is the surrogate regret of the meta algorithm compared to black-box $b_{J^{(i)}}$ on the interval $J^{(i)}$  and $S_{J^{(i)}}^Q(b_{J^{(i)}})$ is the surrogate regret of black-box $b_{J^{(i)}}$ compared to the distribution $Q$. We will now analyze them both.

\subsubsection{Meta Surrogate Regret}\label{subsec:SRm}
The meta surrogate regret is defined as
\[S_{J^{(i)}}^{b_{J^{(i)}}}(\mathcal{M}) := \sum_{t\in J^{(i)}}\Big(\hat{g}_t-g_t(b_{J^{(i)}})\Big).\]
Since $b_{J^{(i)}}$ is active only on $J^{(i)}$ and the learner's loss and the loss of the black-box are equal when the black-box algorithm is not active, we can also write this as follows:
\begin{align*}
S_{J^{(i)}}^{b_{J^{(i)}}}(\mathcal{M}) & = \sum_{t=J_1^{(i)}}^{J_2^{(i)}}\Big(\hat{g}_t-g_t(b_{J^{(i)}})\Big) \\
 & = \sum_{t=1}^T\Big(\hat{g}_t-g_t(b_{J^{(i)}})\Big) \\
 & = S_T^{b_{J^{(i)}}}(\mathcal{M})
 \end{align*}
Now we can apply Lemma \ref{lem:Sreg} where we substitute $\Theta=\mathcal{B}$, $\rho(\theta)=\tau(b)$, $g_t(\theta)=g_t(b)$, $Q=\delta_{b_{J^{(i)}}}$ and $P_t=\tilde{\mathbf{q}}_t$ with $\tau(b)$ the prior on $\mathcal{B}$ and $\delta_b$ the distribution that puts all mass on $b$. Note that we are allowed to use Lemma \ref{lem:Sreg}, since $\tilde{\mathbf{q}}_t$ is determined using Exponential Weights without conditioning. This then gives
\begin{align*}
S_{J^{(i)}}^{b_{J^{(i)}}}(\mathcal{M}) & \leq \text{KL}(\delta_{b_{J^{(i)}}}||\tau) \\
 & = \mathbb{E}_{\delta_{b_{J^{(i)}}}}\ln\Big(\frac{\delta_{b_{J^{(i)}}}(b)}{\tau(b)}\Big) \\
 & = -\ln\tau(b_{J^{(i)}}).
\end{align*}
So we bounded the meta surrogate regret by the negative logarithm of the weight the prior $\tau$ puts on black-box $b_J^{(i)}$. As mentioned in Theorem \ref{thm:Rchange}, we choose $\tau$ to be uniform on $\mathcal{B}$. This yields
\[S_{J^{(i)}}^{b_{J^{(i)}}}(\mathcal{M}) \leq \ln|\mathcal{B}|.\]

\subsubsection{Black-Box Surrogate Regret}\label{subsec:SRb}
Next, we look at the black-box surrogate regret:
\[S_{J^{(i)}}^Q(b_{J^{(i)}}) := \sum_{t\in J^{(i)}}\Big(g_t(b_{J^{(i)}}) - \mathbb{E}_Q[\hat{f}_t(\eta,k)]\Big)\]
This expression is similar to the surrogate regret from Chapter \ref{sec:4}, except that it only covers the interval $J^{(i)}$ instead of all time $[1,T]$. However, since the black-box $b_J^{(i)}$ only operates on $J^{(i)}$ and hence starts running at $t=J_1^{(i)}$, this surrogate regret is the same as the surrogate regret in Chapter \ref{sec:4}, but with $T=|J^{(i)}|$. This enables us to again use Lemma \ref{lem:Sreg}. We then obtain the bound
\[S_{J^{(i)}}^Q(b_{J^{(i)}}) \leq \text{KL}(Q||\gamma\times\pi).\]
Remember that $T$ is known to the learner and can be used by the algorithm, but $I$ and hence the $J^{(i)}$ are unknown and cannot be used. So just like in Chapter \ref{sec:4}, we now set $Q=\delta_{\ddot{\eta}}\times\pi(\cdot|\mathcal{K})$ for some $\ddot{\eta}\in\Gamma=\big\{\frac{1}{2}\big\}\cup\big\{2^{-i}:i=1,2,\dots,\lceil ^2\log{\sqrt{T}}\rceil\big\}$. Moreover, for each black-box $\gamma$ is again the uniform distribution on $\Gamma$ and $\pi$ is again arbitrary. We now find
\[S_{J^{(i)}}^Q(b_{J^{(i)}}) \leq \ln\lceil^2\log\sqrt{T}\rceil-\ln\pi(\mathcal{K}),\]
just like we concluded in Chapter \ref{sec:4}.

\subsubsection{Combined Surrogate Regret}
Recall that in Section \ref{subsec:CBCESquint} we derived that the amount of geometric covering intervals, which form the partitioning of $I$, is bounded by $2\cdot {}^2\log(|I|+2)$. Hence, we can now conclude that the surrogate regret of the learner compared to the distribution $Q$ is bounded by
\begin{align*}
S_I^Q\big(\mathcal{M}(\mathcal{B})\big) & = \sum_{i=-c}^d\Big( S_{J^{(i)}}^{b_{J^{(i)}}}(\mathcal{M}) + S_{J^{(i)}}^Q(b_{J^{(i)}})\Big) \\
 & \leq 2\cdot {}^2\log\big(|I|+2\big) \cdot \big(\ln|\mathcal{B}|+\ln\lceil^2\log\sqrt{T}\rceil-\ln\pi(\mathcal{K})\big)\\
 & \leq \hat{A}_I^{\mathcal{K}}
\end{align*}
where we define
\begin{equation}\label{eq:AIK}
\hat{A}_I^{\mathcal{K}} := \Big(2\cdot {}^2\log\big(|I|+2\big) \cdot \big(\ln|\mathcal{B}|+\ln\lceil^2\log\sqrt{T}\rceil-\ln\pi(\mathcal{K})\big)\Big) \vee 1.
\end{equation}

\subsection{Regular Regret Analysis}\label{subsec:RRA}
Now that we have bounded the surrogate regret of the learner on any interval $I$, we aim to do the same for the regular regret. For this, we introduce a new lemma which is very similar to Lemma \ref{lem:surreg}. Here, we again use the notation $R_I^k\big(\mathcal{M}(\mathcal{B})\big)$ and $V_I^k\big(\mathcal{M}(\mathcal{B})\big)$ for the total regret and variance of the meta algorithm with its black-box algorithms (i.e.\ the learner) on the interval $I$, like introduced in Section \ref{sec:meta-alg}.
\begin{lemma}\label{lem:surregI}
Use any algorithm to determine $P_t^{\mathcal{M}(\mathcal{B})}$. Then for every probability distribution $Q(\eta,k)$ and contiguous interval $I$:
\[\mathbb{E}_Q\big[\eta R_I^k\big(\mathcal{M}(\mathcal{B})\big)\big] \leq \mathbb{E}_Q\big[\eta^2 V_I^k\big(\mathcal{M}(\mathcal{B})\big)\big] + S_I^Q\big(\mathcal{M}(\mathcal{B})\big)\]
\end{lemma}
\begin{proof}[\textbf{Proof of Lemma \ref{lem:surregI}}]
As shown in Section \ref{subsec:SRAgt}, the loss of the learner equals
\begin{align*}
\hat{g}_t & = L\big(\hat{f}_t,P_t^{\mathcal{M}(\mathcal{B})}\big) \\
 & = -\ln\mathbb{E}_{P_t^{\mathcal{M}(\mathcal{B})}}\big[e^{-\hat{f}_t(\eta,k)}\big] \\
 & = -\ln\mathbb{E}_{P_t^{\mathcal{M}(\mathcal{B})}}\big[e^{\eta r_t^k - \eta^2(r_t^k)^2}\big].
\end{align*}
Using the same arguments as in Lemma \ref{lem:surreg}, which remain valid even though we now use $P_t^{\mathcal{M}(\mathcal{B})}$ instead of $P_t$, we find that $\hat{g}_t\geq0$ for all $t$. The definition of $S_I^Q\big(\mathcal{M}(\mathcal{B})\big)$ then yields
\begin{align*}
0 & \leq \sum_{t=I_1}^{I_2}\hat{g}_t \\
 & = S_I^Q\big(\mathcal{M}(\mathcal{B})\big) + \sum_{t=I_1}^{I_2}\mathbb{E}_Q[\hat{f}_t(\eta,k)] \\
 & = S_I^Q\big(\mathcal{M}(\mathcal{B})\big) + \mathbb{E}_Q\big[-\eta R_I^k\big(\mathcal{M}(\mathcal{B})\big) + \eta^2 V_I^k\big(\mathcal{M}(\mathcal{B})\big)\big],
\end{align*}
which implies the desired inequality.
\end{proof}\noindent
For the sake of simplicity, we will use the notation $R_I^{\mathcal{K}}:=R_I^{\mathcal{K}}\big(\mathcal{M}(\mathcal{B})\big)$, $V_I^{\mathcal{K}}:=V_I^{\mathcal{K}}\big(\mathcal{M}(\mathcal{B})\big)$ and $S_I^Q:=S_I^Q\big(\mathcal{M}(\mathcal{B})\big)$ for the remainder of this chapter. We now complete the proof of Theorem \ref{thm:Rchange}.
\begin{proof}[\textbf{Proof of Theorem \ref{thm:Rchange}}]
As mentioned in Section \ref{subsec:SRb}, we again set $Q=\delta_{\ddot{\eta}}\times\pi(\cdot|\mathcal{K})$ for some $\ddot{\eta}\in\Gamma=\big\{\frac{1}{2}\big\}\cup\big\{2^{-i}:i=1,2,\dots,\lceil ^2\log{\sqrt{T}}\rceil\big\}$, let $\gamma(\eta)$ be the uniform distribution on $\Gamma$ and let $\pi(k)$ be arbitrary. Finally, let $\mathcal{B}=\{b_J:J=[J_1,J_2]\in\mathcal{J}\text{ with }J_2\leq T\}$ be the set of black-box algorithms running only on a geometric covering interval with endpoint not larger than $T$ and let $\tau(b)$ be the uniform distribution on $\mathcal{B}$. Using the earlier obtained bound on the surrogate regret, Lemma \ref{lem:surregI} translates to
\begin{align*}
R_I^{\mathcal{K}} & \leq \ddot{\eta}V_I^{\mathcal{K}} + \frac{1}{\ddot{\eta}}S_I^Q \\
 & \leq \ddot{\eta}V_I^{\mathcal{K}} + \frac{1}{\ddot{\eta}}\hat{A}_I^{\mathcal{K}}
\end{align*}
with $\hat{A}_I^{\mathcal{K}}$ defined as in (\ref{eq:AIK}). We now apply the exact same optimization steps for $\ddot{\eta}$ as in the proof of Theorem \ref{thm:Rnochange} to obtain
\begin{align*}
R_I^{\mathcal{K}} & \leq 2\sqrt{2V_I^{\mathcal{K}}\hat{A}_I^{\mathcal{K}}}+4\hat{A}_I^{\mathcal{K}}. \\
\end{align*}
Lastly, we want to gain insight in the size of the term $\ln|\mathcal{B}|$ in the expression of $\hat{A}_I^{\mathcal{K}}$. To determine this, we return to the visualization of the geometric covering intervals, given in Figure \ref{fig:JVis} from Chapter \ref{sec:3}, which we display again here.
\noindent
\FloatBarrier
\begin{figure}[h]
    \renewcommand\thefigure{\ref{fig:JVis}}
    \centering
    \includegraphics[scale=0.4]{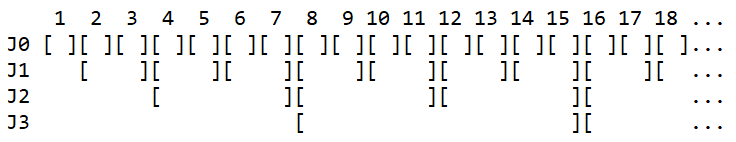}
    \caption{Geometric covering intervals. Each interval is denoted by [ ]. Adopted from \cite{ChEnv}.}
\end{figure}
\FloatBarrier
\noindent
We only look at geometric covering intervals with endpoint $J_2\leq T$. Then there are $T$ intervals of length 1, $\lfloor\frac{T-1}{2}\rfloor$ intervals of length 2, $\lfloor\frac{T-3}{4}\rfloor$ intervals of length 4 and so on. For each interval there exists an $i\in\mathbb{N}\cup\{0\}$ such that the interval length equals $2^i$. Since the first interval of the set of intervals with length $2^i$ has starting point $2^i$ and endpoint $2^{i+1}-1$, we find that $2^{i+1}-1\leq T$. This leads to the conclusion that the length $2^i$ of every interval, used by black-boxes from $\mathcal{B}$, is bounded by $2^i\leq\frac{T+1}{2}$ and thus $i\leq\lfloor {}^2\log\frac{T+1}{2}\rfloor$. This gives
\begin{align*}
|\mathcal{B}| & = \sum_{i=0}^{\lfloor ^2\log \frac{T+1}{2} \rfloor}\Big\lfloor\frac{T-2^i+1}{2^i}\Big\rfloor \\
 & = \sum_{i=0}^{\lfloor ^2\log \frac{T+1}{2} \rfloor}\Big\lfloor\frac{T+1}{2^i}-1\Big\rfloor \\
 & \leq \sum_{i=0}^{\lfloor ^2\log \frac{T+1}{2} \rfloor}\frac{T+1}{2^i}-\Big\lfloor {}^2\log \frac{T+1}{2} \Big\rfloor-1 \\
 & = (T+1)\cdot\frac{1-\frac{1}{2}^{\lfloor ^2\log \frac{T+1}{2}\rfloor+1}}{1-\frac{1}{2}}-\Big\lfloor {}^2\log \frac{T+1}{2} \Big\rfloor-1 \\
 & = (T+1)\cdot(2-2^{-\lfloor ^2\log \frac{T+1}{2} \rfloor})-\Big\lfloor {}^2\log \frac{T+1}{2} \Big\rfloor-1 \\
 & \leq (T+1)\cdot\Big(2- \frac{2}{T+1}\Big) - \Big\lfloor{}^2\log\frac{T+1}{2}\Big\rfloor-1 \\
 & = 2T-\Big\lfloor{}^2\log\frac{T+1}{2}\Big\rfloor-1 \\
 & \leq 2T
\end{align*}
which completes the proof of Theorem \ref{thm:Rchange}.
\end{proof}\noindent

\subsection{Improving the Bound}\label{subsec:newtau}
We finish this chapter by proposing another improvement of the bound, which changes the term $\ln T$. While deriving (\ref{eq:RIKbbound}) we chose the prior $\tau(b)$ to be equal to the uniform distribution on $\mathcal{B}$. Choosing a different prior could give a better bound, one possibility being the prior Jun et al.\ \cite{ChEnv} use for their meta algorithm CBCE, which we discussed in Chapter \ref{sec:3}:
\begin{equation}\label{eq:taujun}
\tau(b_J) = Z^{-1}\cdot\frac{1}{J_1^2\big(1+\lfloor ^2\log{J_1}\rfloor\big)} \text{ for all } b_J\in\mathcal{B}
\end{equation}
with $Z$ a normalization factor, $J_1$ the starting point of the interval $J$ and $\mathcal{B}$ the set of black-box algorithms operating only on a geometric covering interval with endpoint not larger than $T$. The following theorem provides the bound when using this prior.
\begin{theorem}
\label{thm:Rchangebonus}
Let $\mathcal{B} = \{b_J:J=[J_1,J_2]\in\mathcal{J}\text{ with }J_2\leq T\}$ be the set of black-box algorithms operating solely on a geometric covering interval with endpoint not larger than $T$. Furthermore, let $\tau(b)$ be as defined in (\ref{eq:taujun}), let $\gamma(\eta)$ be the uniform distribution on $\Gamma=\big\{\frac{1}{2}\big\}\cup\big\{2^{-i}:i=1,2,\dots,\lceil ^2\log{\sqrt{T}}\rceil\big\}$ and let $\pi(k)$ be arbitrary. Then for any contiguous interval $I=[I_1,I_2]\subseteq\{1,2,\dots,T\}$, Algorithm \ref{alg:SCE} yields
\[R_I^{\mathcal{K}}\leq2\sqrt{2V_I^{\mathcal{K}}\tilde{A}_I^{\mathcal{K}}}+4\tilde{A}_I^{\mathcal{K}}\]
with
\[\tilde{A}_I^{\mathcal{K}} := \bigg(2\cdot {}^2\log\big(|I|+2\big) \cdot \Big(\frac{1}{2}+3\ln{I_2}+\ln\lceil^2\log\sqrt{T}\rceil-\ln\pi(\mathcal{K})\Big)\bigg) \vee 1\]
where $x \vee y := \max\{x,y\}$.
\end{theorem}\noindent
Theorem \ref{thm:Rchangebonus} results in the bound
\begin{equation}\label{eq:RIKbboundImpr}
\begin{aligned}
R_I^{\mathcal{K}} & \preccurlyeq \sqrt{\ln|I|\cdot V_I^{\mathcal{K}}\big(\ln I_2 + \ln(\ln T)-\ln\pi(\mathcal{K})\big)} \\
 & \qquad + \ln|I|\cdot \big(\ln I_2 + \ln(\ln T)-\ln\pi(\mathcal{K})\big)
\end{aligned}
\end{equation}
which is a small improvement on the previously obtained bound (\ref{eq:RIKbbound}) as the term $\ln{T}$ is replaced by $\ln{I_2}$ and the inequality $I_2\leq T$ holds.
\begin{proof}[\textbf{Proof of Theorem \ref{thm:Rchangebonus}}]
The proof follows the exact same steps as the proof of Theorem \ref{thm:Rchange}. The only difference is the bound of the meta surrogate regret. Remember that we only use black-box algorithms, which run solely on geometric covering intervals. In Chapter \ref{sec:3} we concluded that at time $t$ there are $1+\lfloor ^2\log{t}\rfloor$ active intervals, which implies that there are at most $1+\lfloor ^2\log{t}\rfloor$ intervals with starting point $t$. This allows us to bound the normalization factor $Z$:
\begin{align*}
Z & = \sum_{b_J\in\mathcal{B}}\frac{1}{J_1^2\big(1+\lfloor ^2\log{J_1}\rfloor\big)} \\
 & \leq \sum_{t=1}^{\infty}\frac{1+\lfloor ^2\log{t}\rfloor}{t^2\big(1+\lfloor ^2\log{t}\rfloor\big)} \\
 & = \sum_{t=1}^{\infty}\frac{1}{t^2} \\
 & = \frac{\pi^2}{6}
\end{align*}
This gives
\[\tau(b_J) \geq \frac{6}{\pi^2}\cdot\frac{1}{J_1^2\big(1+\lfloor ^2\log{J_1}\rfloor\big)}.\]
Using the bound on the meta surrogate regret of Section \ref{subsec:SRm}, we obtain
\begin{align*}
S_{J^{(i)}}^{b_{J^{(i)}}}(\mathcal{M}) & \leq -\ln\tau(b_{J^{(i)}}) \\
 & \leq \ln\Big(\frac{\pi^2}{6}\Big) + \ln\Big(\big(J_1^{(i)}\big)^2\cdot\big(1+\lfloor{}^2\log{J_1^{(i)}}\rfloor\big)\Big) \\
 & \leq \frac{1}{2} + 3\ln(I_2).
\end{align*}
The last inequality follows from the fact that all $J^{(i)}$ are subsets of $I$, which gives $J_1^{(i)}\leq I_2$ with $I_2$ the endpoint of $I$. Moreover, the inequality $1+\lfloor{}^2\log t \rfloor \leq t$ holds for all $t\geq 1$, since $1+{}^2\log 1 = 1$, $1+{}^2\log 2 = 2$ and $1+{}^2\log t$ grows slower than $t$ for $t\geq 2$. This follows from the derivative of $1+{}^2\log t$, which equals $\frac{1}{t\ln 2}$ and is smaller than $1$, the derivative of $t$, for all $t\geq2$. So now, using this newly obtained bound on the meta surrogate regret, we can replace the term $\ln|\mathcal{B}|$ by $\frac{1}{2}+3\ln(I_2)$ in every expression. All the other steps for obtaining the bound on the regular regret stay the same. This concludes the proof.
\end{proof}\noindent
As mentioned before, using this new prior yields a small improvement of the bound. Possibly, one can find a prior which gives an even better bound, but we leave this for future research. In conclusion, we have now developed an algorithm which maintains Squint's advantages, mentioned in Section \ref{subsec:2.3}, in a changing environment and is an improvement on the CBCE + Squint algorithm from Chapter \ref{sec:3}.

\newpage
\section{Conclusion}\label{sec:6}
In this thesis we studied the prediction with expert advice setting. We first looked at a non-changing environment, i.e.\ where the best expert does not change over time. A well-known algorithm, named Hedge, yielded the regret bound
\[R_T^k \preccurlyeq \sqrt{T\ln{K}}\]
while Squint obtained the bound
\[R_T^{\mathcal{K}} \preccurlyeq \sqrt{V_T^{\mathcal{K}}\big(\ln(\ln{T})-\ln{\pi(\mathcal{K})}\big)} + \ln(\ln{T})-\ln{\pi(\mathcal{K})}.\]
We concluded that Squint had an advantage over Hedge, because
\begin{itemize}
\item the Squint regret bound contains the variance $V_T^{\mathcal{K}}$ instead of the time $T$
\item the Squint regret bound uses $\frac{1}{\pi(\mathcal{K})}$, the inverse of the prior weight on a set of good experts $\mathcal{K}$, while the Hedge bound contains $K$, the total number of experts
\end{itemize}
Both factors for Squint are never larger than those for Hedge when one chooses the prior $\pi(k)$ on the experts to be uniform. In specific cases, the values of those factors are even much lower.\par
When moving to a changing environment, we want to bound the regret on every interval and not only from $t=1$ to $T$. The common solution to obtain a good regret bound on any interval $I=[I_1,I_2]$ in a changing environment is to use the original algorithm on different intervals. These algorithms on different intervals then are called black-box algorithms. One then applies a meta algorithm on these black-boxes to learn the best intervals and hence which black-box algorithms to follow. However, when we applied a meta algorithm, called Coin Betting for Changing Environment, the advantage Squint had over Hedge vanished. Applying CBCE to Hedge gave
\[R_I^k \preccurlyeq \sqrt{|I|\ln{I_2}}+\sqrt{|I|\ln{K}}\]
while doing the same for Squint gave
\begin{align*}
R_I^{\mathcal{K}} & \preccurlyeq \sqrt{|I|\ln{I_2}} + \sqrt{\ln|I|\cdot V_I^{\mathcal{K}}\big(\ln(\ln{|I|})-\ln\pi(\mathcal{K})\big)}\\
& \qquad + \ln|I|\cdot\big(\ln(\ln|I|)-\ln\pi(\mathcal{K})\big)
\end{align*}
where $|I|$ is the length of $I$ and $I_2$ is the endpoint of $I$. In both cases the overhead of CBCE to learn the best interval, $\sqrt{|I|\ln{I_2}}$, is the dominating term. This makes Squint's advantages in comparison to Hedge disappear. We wanted to find a way to retain Squint's properties, but applying a meta algorithm would not help. Hence, we looked into the construction of Squint and the proof of its regret bound in a non-changing environment. We saw that the creators of Squint used a reduction of the original task, for which Squint was made, to a surrogate task in order to prove the bound. They used the Exponential Weights algorithm for the surrogate task, which then yielded Squint for the original task. \par
We used this approach of making a reduction and combined it with the idea of using black-box algorithms with a meta algorithm in a changing environment. This way we created an algorithm for the surrogate task in a changing environment. We then obtained our own algorithm, named Squint-CE, for the original task in a changing environment. Its pseudocode is expressed by Algorithm \ref{alg:SCE}. We found that the regret bound of Squint-CE was significantly better than the one of CBCE applied to Squint:
\begin{align*}
R_I^{\mathcal{K}} & \preccurlyeq \sqrt{\ln|I|\cdot V_I^{\mathcal{K}}\big(\ln I_2 + \ln(\ln T)-\ln\pi(\mathcal{K})\big)} \\
 & \qquad + \ln|I|\cdot \big(\ln I_2 + \ln(\ln T)-\ln\pi(\mathcal{K})\big)
\end{align*}
We only had to pay a term $\ln I_2$ and a factor $\ln|I|$ with respect to the Squint bound in a non-changing environment, whereas for the combination of CBCE with Squint we obtained a dominating term $\sqrt{|I|\ln I_2}$. In conclusion, we managed to retain Squint's advantages in a changing environment while only having to pay a small price.

\newpage
\cleardoublepage
\phantomsection
\addcontentsline{toc}{section}{References}

\begin{thebibliography}{12}
\providecommand{\natexlab}[1]{#1}
\providecommand{\url}[1]{\texttt{#1}}
\expandafter\ifx\csname urlstyle\endcsname\relax
  \providecommand{\doi}[1]{doi: #1}\else
  \providecommand{\doi}{doi: \begingroup \urlstyle{rm}\Url}\fi

\bibitem[Adamskiy et~al.(2012)Adamskiy, Koolen, Chernov, and Vovk]{EquivLoss}
Dmitry Adamskiy, Wouter~M. Koolen, Alexey Chernov, and Vladimir Vovk.
\newblock A closer look at adaptive regret.
\newblock In Nader~H. Bshouty, Gilles Stoltz, Nicolas Vayatis, and Thomas
  Zeugmann, editors, \emph{Algorithmic Learning Theory}, pages 290--304,
  Berlin, Heidelberg, 2012. Springer Berlin Heidelberg.
\newblock ISBN 9783642341069.

\bibitem[Banerjee(2006)]{CompLem}
Arindam Banerjee.
\newblock On bayesian bounds.
\newblock In \emph{Proceedings of the 23rd International Conference on Machine
  Learning}, ICML ’06, pages 81--88, New York, NY, USA, 2006. Association for
  Computing Machinery.
\newblock ISBN 1595933832.
\newblock \doi{10.1145/1143844.1143855}.
\newblock URL \url{https://doi.org/10.1145/1143844.1143855}.

\bibitem[Cesa-Bianchi and Lugosi(2006)]{Backgr1}
Nicolò Cesa-Bianchi and Gábor Lugosi.
\newblock \emph{Prediction, Learning, and Games}.
\newblock Cambridge University Press, Cambridge, 2006.
\newblock ISBN 9780521841085.
\newblock \doi{10.1017/CBO9780511546921}.

\bibitem[Cesa-Bianchi et~al.(2007)Cesa-Bianchi, Mansour, and Stoltz]{Ineq}
Nicolò Cesa-Bianchi, Yishay Mansour, and Gilles Stoltz.
\newblock Improved second-order bounds for prediction with expert advice.
\newblock \emph{Machine Learning}, 66\penalty0 (2-3):\penalty0 321--352, 2007.
\newblock URL
  \url{https://link.springer.com/article/10.1007/s10994-006-5001-7}.

\bibitem[Daniely et~al.(2015)Daniely, Gonen, and Shalev-Shwartz]{ChEnvPr}
Amit Daniely, Alon Gonen, and Shai Shalev-Shwartz.
\newblock Strongly adaptive online learning.
\newblock In Francis Bach and David Blei, editors, \emph{Proceedings of the
  32nd International Conference on Machine Learning}, volume~37 of
  \emph{Proceedings of Machine Learning Research}, pages 1405--1411, Lille,
  France, 07--09 Jul 2015. PMLR.
\newblock URL \url{http://proceedings.mlr.press/v37/daniely15.html}.

\bibitem[Freund and Schapire(1997)]{Hedge}
Yoav Freund and Robert~E. Schapire.
\newblock A decision-theoretic generalization of on-line learning and an
  application to boosting.
\newblock \emph{Journal of Computer and System Sciences}, 55\penalty0
  (1):\penalty0 119--139, 1997.
\newblock ISSN 0022-0000.
\newblock \doi{https://doi.org/10.1006/jcss.1997.1504}.
\newblock URL
  \url{http://www.sciencedirect.com/science/article/pii/S002200009791504X}.

\bibitem[Jun et~al.(2017)Jun, Orabona, Wright, and Willett]{ChEnv}
Kwang-Sung Jun, Francesco Orabona, Stephen Wright, and Rebecca Willett.
\newblock Improved strongly adaptive online learning using coin betting.
\newblock In Aarti Singh and Jerry Zhu, editors, \emph{Proceedings of the 20th
  International Conference on Artificial Intelligence and Statistics},
  volume~54 of \emph{Proceedings of Machine Learning Research}, pages 943--951,
  Fort Lauderdale, FL, USA, 20--22 Apr 2017. PMLR.
\newblock URL \url{http://proceedings.mlr.press/v54/jun17a.html}.

\bibitem[Koolen and Erven(2015)]{Squint}
Wouter~M. Koolen and Tim~Van Erven.
\newblock Second-order quantile methods for experts and combinatorial games.
\newblock In Peter Grünwald, Elad Hazan, and Satyen Kale, editors,
  \emph{Proceedings of The 28th Conference on Learning Theory}, volume~40 of
  \emph{Proceedings of Machine Learning Research}, pages 1155--1175, Paris,
  France, 03--06 Jul 2015. PMLR.
\newblock URL \url{http://proceedings.mlr.press/v40/Koolen15a.html}.

\bibitem[Koolen et~al.(2016)Koolen, Gr\"{u}nwald, and van Erven]{SquintUseful}
Wouter~M. Koolen, Peter Gr\"{u}nwald, and Tim van Erven.
\newblock Combining adversarial guarantees and stochastic fast rates in online
  learning.
\newblock In \emph{Proceedings of the 30th International Conference on Neural
  Information Processing Systems}, NIPS’16, page 4464–4472, Red Hook, NY,
  USA, 2016. Curran Associates Inc.
\newblock ISBN 9781510838819.

\bibitem[Orabona and P\'{a}l(2016)]{CoinB}
Francesco Orabona and D\'{a}vid P\'{a}l.
\newblock Coin betting and parameter-free online learning.
\newblock In \emph{Proceedings of the 30th International Conference on Neural
  Information Processing Systems}, NIPS’16, page 577–585, Red Hook, NY,
  USA, 2016. Curran Associates Inc.
\newblock ISBN 9781510838819.

\bibitem[Shalev-Shwartz(2012)]{Backgr2}
Shai Shalev-Shwartz.
\newblock Online learning and online convex optimization.
\newblock \emph{Foundations and Trends® in Machine Learning}, 4\penalty0
  (2):\penalty0 107--194, 2012.
\newblock ISSN 1935-8237.
\newblock \doi{10.1561/2200000018}.
\newblock URL \url{http://dx.doi.org/10.1561/2200000018}.

\bibitem[van~der Hoeven et~al.(2018)van~der Hoeven, van Erven, and
  Kot{\l}owski]{ExpW}
Dirk van~der Hoeven, Tim van Erven, and Wojciech Kot{\l}owski.
\newblock The many faces of exponential weights in online learning.
\newblock In S\'ebastien Bubeck, Vianney Perchet, and Philippe Rigollet,
  editors, \emph{Proceedings of the 31st Conference On Learning Theory},
  volume~75 of \emph{Proceedings of Machine Learning Research}, pages
  2067--2092. PMLR, 06--09 Jul 2018.
\newblock URL \url{http://proceedings.mlr.press/v75/hoeven18a.html}.

\end{thebibliography}

\end{document}